\documentclass{article}

 \usepackage{multirow}
\usepackage[table,xcdraw]{xcolor}

\usepackage[nonatbib, preprint]{neurips_2022}
\usepackage[utf8]{inputenc} 
\usepackage[T1]{fontenc}    
\usepackage{hyperref}       
\usepackage{url}            
\usepackage{macros}
\usepackage{booktabs}       
\usepackage{amsfonts}       
\usepackage{nicefrac}       
\usepackage{microtype}      
\usepackage{xcolor}         
\usepackage[textsize=tiny,colorinlistoftodos]{todonotes}
\usepackage{adjustbox}
\usepackage{floatrow}
\usepackage{amsmath}
\usepackage{amssymb}
\usepackage{mathtools}
\usepackage{amsthm}
\usepackage{caption}
\usepackage{subcaption}
\usepackage{enumitem}
\usepackage{float}

\usepackage[capitalize,noabbrev]{cleveref}
\usepackage{bbm}

\theoremstyle{plain}
\newcommand{\m}{\mathcal}
\renewcommand{\a}{\alpha}
\renewcommand{\b}{\beta}
\newcommand{\xhdr}[1]{\vspace{0.1mm}\noindent{{\bf #1.}}}
\newfloatcommand{capbtabbox}{table}[][\FBwidth]

\title{Tuning the Geometry of Graph Neural Networks}

\author{%
  Sowon Jeong \\
  Department of Statistics\\
University of Chicago\\
  Chicago, IL, 60637 \\
  \texttt{sowonjeong@uchicago.edu} \\
   \And
Claire Donnat \\
University of Chicago\\
  Chicago, IL, 60637 \\
   \texttt{cdonnat@uchicago.edu} \\
}

\begin{document}

\maketitle

\begin{abstract}

By recursively summing node features over entire neighborhoods, spatial graph convolution operators have been heralded as key to the success of Graph Neural Networks (GNNs). Yet, despite the multiplication of GNN methods across tasks and applications, the impact of this aggregation operation on their performance still has yet to be extensively analyzed. In fact, while efforts have mostly focused on optimizing the architecture of the neural network, fewer works have attempted to characterize  \textit{(a)  the different classes of spatial convolution operators}, {\it(b) how the choice of a particular class relates to properties of the data }, and {\it (c) its impact on the geometry of the embedding space}. In this paper, we propose to answer all three questions by dividing existing operators into two main classes ({\it symmetrized} vs. {\it row-normalized} spatial convolutions), and show how these translate into different implicit biases on the nature of the data. Finally, we show that this aggregation operator is in fact tunable, and explicit regimes in which certain choices of operators --- and therefore, embedding geometries --- might be more appropriate. 

\end{abstract}
\vspace{-0.3cm}
\section{Introduction and Motivation}
 As the extension of the Deep Neural Network (DNN) machinery to the graph setting, Graph Neural Networks (GNN) offer a powerful paradigm for extending Machine Learning tools to the analysis of relational data, typically modelled through a graph \cite{battaglia2018relational,dong2020graph, hamilton2017inductive, kipf2016semi, wu2020comprehensive,zhou2020graph}. The recent impressive success of GNNs across tasks and applications \cite{casas2019spatially,gaudelet2021utilizing,gilmer2017neural,li2021representation,ma2019genn,wu2020graph} in dealing with this complex data type has been attributed to their two main components (see Figure~\ref{fig:Schema}) : \textbf{(a) a convolution operator} $\mc{C}$ (or propagation operator \cite{zhou2020graph,zhou2020understanding}) that aggregates information contained in the neighborhood of any given node to create neighborhood-aware embeddings; and \textbf{ b) a  non-linear layer} --- or transformator\cite{zhou2020understanding} ---, that multiplies the convolved features by a weight matrix before applying a non-linearity. All weight and bias parameters of the GNN are typically trained in an end-to-end fashion and, as for Deep Neural Networks, have been deemed essential in creating powerful {non-linear node embeddings} that can be tailored to any downstream prediction task. 
Depending on the architecture of the network, such graph convolution blocks are then stacked and/or repeated to encode varying ``receptive depths'' \cite{frasca2020sign, kipf2016semi,wu2019simplifying}. More formally, denoting as $H^{(k)}$  the output of the $k^{th}$ layer, GNNs can be broadly understood as a pipeline of sequential node convolutions of the form $H^{(k)}_u = \sigma(\mc{C}_{\mc{N}(u)}( H^{(k-1)}) W_k  + b_k)$,
where $H^{(0)} = X$ are the nodes' raw features, $\sigma$ is a non-linearity (e.g. ReLU), and $\mc{C}_{\mc{N}(u)}$ is the convolution operator applied to the neighborhood $\mc{N}(u)$ of node $u.$
The final layer is always taken to be linear and written as:
\vspace{-0.15cm}
\begin{gather} \label{eq:last_layer}
    H_u^{(K)} = \mc{C}_{\mc{N}(u)}(H^{(K-1)})  W^{(K)}  + b^{(K)}.
\end{gather}
Throughout this paper, we interpret the quantity $ \mc{C}_{\mc{N}(u)}( H^{(K-1)})$ as our node embeddings, so that the last layer can be understood as a  generalized linear model \cite{mccullagh2019generalized} operating on the transformed input features $ \mc{C}_{\mc{N}(u)}( H^{(K-1)})$.


\begin{figure}[h]
    \centering
    \includegraphics[width=14cm]{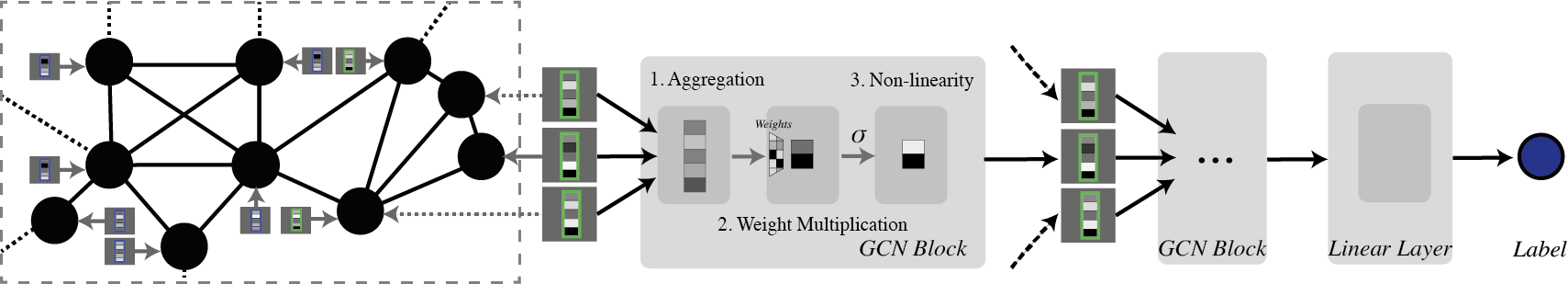}
    \caption{Graph Neural Network Block: Illustration of the aggregation (convolution) and transformation steps within each GNN block. GNNs blocks are typically concatenated to allow information to percolate from all parts of the network. }
    \label{fig:Schema}
\end{figure}

\xhdr{The convolution operator} Most existing theoretical analyses of GNNs draw a distinction between two main classes of  convolution operators $\mc{C}$\cite{zhou2020graph}:  {\it spectral operators} \cite{defferrard2016convolutional, dong2020graph,gama2018convolutional,gama2018diffusion,henaff2015deep}, that act of the eigenvectors of some version of the Graph Laplacian (defined in its unnormalized version as $L= D-A$, with $A$ the adjacency matrix and $D$ the degree matrix), and {\it spatial operators}, that recursively aggregate node features within a given neighborhood. While the dichotomy between the two is in practice often attenuated by their implementation (spectral methods usually resort to using low-order Chebychev polynomials of the Laplacian \cite{defferrard2016convolutional,shuman2011chebyshev,shuman2013emerging}, thus effectively ``localizing'' the signal), these two types of approaches have different interpretations and downstream consequences for the analysis of the data. Consequently, in this paper, we propose to focus on spatial convolutions, as popularized by the framework of Kipf et al \cite{kipf2016semi}. In this case, for each source node $u$, the aggregation operator usually consists in summing nodes features over its entire neighborhood $\mc{N}(u)$, so that the convolution becomes the function: $\mc{C}: \{X_v\}_{v\in \mc{N}(u)} \to \mc{C}(\{X_v\}_{v\in \mc{N}(u)})=SX$, where $S$ is a weight matrix. Kipf et al \cite{kipf2016semi} suggest taking $S$ to be  normalized adjacency matrix with added self-loops:  $S =  \hat{D}^{-1/2} \hat{A} D^{-1/2}$, where  $\hat{A}=A+I$  is the adjacency matrix of the graph with self-loops, $\hat{D}$ its diagonal degree matrix, and $X\in \mathbb{R}^{n \times p}$ is the feature matrix for the $n$ nodes.  The summation is crucial in preserving permutation equivariance of the neighborhood, and has been promoted as  being able to construct features that encode long-range dependencies: by repeating and stacking multiple GNN layers, information can percolate at various distances through the graph --- thus efficiently capturing information contained in both the nodes' features and the graph structure.


\xhdr{Variations on the convolution operator} While GNN methods have grown increasingly popular, the spatial convolution operator seems to have received a little less attention from the community. Table~\ref{tab:operators_all} in Appendix \ref{appendix:gnn} provides a summary of the main graph convolution operators that are currently available in the Pytorch geometric package \cite{fey2019fast} --- here taken as a proxy for the most popular convolution choices. As observed from this table, all of these operators rely on a (weighted) sum of neighborhood features, in line with the framework of Kipf et al \cite{kipf2016semi}. In fact, while GraphSage \cite{hamilton2017inductive} has been attempted using ``max'' pooling and an LSTM version of the convolution operator, the authors did not report any significant advantage of these methods over the simple sum. Similarly, in \cite{xu2018powerful}, the authors show that a simple summation is often sufficient to characterize a multiset. Thus, without too much loss of generality, we restrict our study to sum-based aggregators. In this setting, two additional main variations have appeared over the recent years: attention-based spatial convolutions (\cite{brody2021attentive,velivckovic2017graph, xie2020mgat}), which attempt to learn edge weights tailored the task at hand, and more broadly, weighted spatial convolutions (e.g. using graphon weights, \cite{parada2021graphon,ruiz2020graphon}, kernels \cite{nikolentzos2020random,feng2022kergnns} among others\cite{zhang2020wagnn}), that refine the adjacency matrix by endowing it with edge intensities. While these efforts have focused on defining ``the best edge weights'', little attention has been put on the type of convolutions that these yield. Yet, attention-based convolutions tend to be {\it row-normalized}, while the latter usually suggest weighted {\it symmetrized} adjacencies. Our purpose here is to show that this choice is not innocuous, and translates into fundamental differences in the geometry of the embedding space.

\xhdr{Prior work: studying the effect of the Convolution operator} The literature focusing on understanding the inner mechanisms behind the success of Graph Neural Networks has considerably expanded over the past few years. Most notably, an important body of theoretical work has focused on understanding the role of this convolution operator in phenomena such as oversmoothing \cite{oono2019graph,cai2020note,chen2020measuring} and oversquashing\cite{alon2020bottleneck,topping2021understanding}. In this setting, most of the analyses are led from a spectral perspective, relating the behavior of the embeddings to properties of the eigenvalues of the convolution operator.  However,to the best of our knowledge, none of these approaches have specifically focused on understanding the effect of the convolution operator on the underlying geometry of these embeddings, nor has attempted to tie these properties to topological properties of the underlying nodes.

\xhdr{Contributions} The objective of this paper is to answer three questions:\textit{(a) how does the choice of a particular convolution operator affect the organisation of the embedding space},  \textit{(b) how does it relate to the original properties (ie features and topological distances)  of the graph}, and {\it (c) what is the most appropriate convolution operator for a given dataset?} We will attempt and answer all three questions by studying two larger families of row-normalized and symmetrized convolution operators (parametrized by the variables $\alpha \in [0,1]$ and  $\beta \in \R^+$), allowing to show how the convolution operator itself is in fact tunable.  In particular, we will show different values of $\a$ and $\b$ impact the organization of the latent space (Section \ref{sec:geometry}) and the inherent geometry of the embeddings (Section \ref{sec:geometry-intrinsic}), and characterize regimes in which certain choices of operators might be more relevant than others.


\vspace{-0.3cm}
\section{Defining a family of convolution operators}\label{sec:convolutions}


To analyse the impact of the convolution operator on the embedding geometry, we define two main families of spatial convolutions:
\begin{description}[topsep=-0.5em, leftmargin=5mm]
\item[a. Symmetrized Convolutions Operators,] defined as the family of operators of the form: 
\begin{align}\label{eq:normalized}
 \mathcal{F} =\big\{ M_{\alpha, \beta} = D_{\beta}^{-\alpha} (A + \beta I) D_{\beta}^{-\alpha} \big| {\alpha \in [0,1], \beta \in [0, + \infty),  D_{\beta} =\text{diag}\big((A + \beta I) \mathbbm{1})\big) }\big\}. 
 \end{align}
Here, $ D_{\beta}$ is the degree matrix associated with the $\beta$-augmented adjacency matrix $A+\beta I$.
Note that this family is a generalization of the traditional GCN convolution $S_{GCN} = \hat{D}^{-1/2}\hat{A}\hat{D}^{-1/2}$, which corresponds here to a choice of $\alpha=0.5$ and $\beta=1$. The choice $\alpha=0.5$ and $\beta=2$ also features amongst the default implementations in Pytorch Geometric \cite{fey2019fast}. Similarly, the convolution operator indexed by  $\alpha=0, \beta=1$ corresponds to a version of the GIN convolution \cite{xu2018powerful}, and more generally, the sum-based message-passing versions of GNNs \cite{battaglia2018relational}  --- another popular set of choices for GNN algorithms. 
\item[b. Row-normalized Convolution operators], which we define the general family of the form: 
\begin{align}\label{eq:regularized}
\mathcal{M} =\big\{ D_{(\alpha, \beta)}^{-1} M_{\alpha, \beta}, \qquad \text{with } D_{(\alpha, \beta)}^{-1} = \text{diag}\big(M_{\alpha, \beta} \mathbbm{1})\big) \big|  M_{\alpha, \beta} \in \mathcal{F}\big\} 
\end{align}
This family encompasses a number of operators, such as the sum-based convolution deployed in GraphSage\cite{hamilton2017inductive} --- and, to some extent, that of GAT  \cite{velivckovic2017graph}, which considers row-normalized convolutions on a modified version of the adjacency matrix.
\end{description}
For both families of convolution operators, the parameter $\a$ impact the weight assigned to nodes with high degree: as $\alpha$ increases, high-degree nodes are increasingly penalised, so that their contribution to neighbouring node embeddings decreases (relatively to lower degree nodes). On the other hand, $\beta$ can be interpreted as capturing the amount of ``innovation'' or relevant information that the source node brings to the embedding with respect to its neighbourhood. In particular, for high values of $\beta$, the source node's contribution to the node overweighs that of the neighborhood, so the embedding is essentially identical to the source node's. Consequently,  $\beta$ allows us to interpolate between the traditional GNN regime ($\beta=1$) and the MLP setting.

\xhdr{Empirical consequences of a choice of operator} While seemingly specious, these families inherently make different assumptions on the nature of the data. More precisely, row-normalized convolutions create new embeddings that are convex combinations of the neighbors. This implicitly assumes that the data lives on some smooth Riemannian manifold, where the degree or centrality of a node corresponds to an assumption on the sampling distribution over $\mathcal{M}$: nodes with high degree correspond to ``well sampled'' areas of the manifold.  Therefore, in this setting, it is intuitively possible to get an accurate representation of the local information by averaging neighbourhood features:  the higher the degree, the higher the amount of certainty around the node's value. By contrast, symmetrized embeddings are weighted sums --- but not convex combinations --- of neighbours. Here, the sum simply plays the role of a permutation-invariant aggregation operator\cite{hamilton2020graph}, and as we will see, is able to encode topological features (e.g. structural roles \cite{donnat2018learning}). 

\xhdr{Experiments} Finally, to motivate our study before diving into more theoretical considerations, we propose to highlight the impact of the choice using standard benchmarks in the literature (we refer the reader to Appendix \ref{appendix:experiments} for an overview of the properties of these datasets).  Figure~\ref{fig:introb} highlights the impact of the value of  $\alpha$ and $\beta$ on the classification accuracy for the Cora dataset. In particular, for the symmetric operator, the value of $\alpha$ seems to drive the ``first order'' difference, and exhibits the potential to vastly influence the performance. By contrast, the choice of $\beta$ seems to affect the performance less --- unless $\beta$ becomes too big and overweighs the rest of the neighbors. This effect might be due to the significant level of homophily in the Cora dataset: in this setting, the source node's feature vector is fairly redundant with that of its neighbours. However, we show in Appendix \ref{appendix:experiments} additional examples where the impact of $\beta$ is much more substantial.  Most strikingly, the choice of $\alpha$ seems to have a significant effect on the performance of the symmetrized GCN, with a phase transition at $\alpha=0.5$: for values of $\a$ greater than this threshold, the performance drops quite substantially. Noticeably, in the ``poor'' performance region, the interaction between $\alpha$ and $\beta$ is more marked: choosing low value of $\b$  (i.e. $\beta=0$) seems to mitigate the decrease in performance. 
\begin{figure}
     \centering
          \begin{subfigure}[t]{0.37\textwidth}
              \centering
    \includegraphics[width=1.0\textwidth, height=3.8cm]{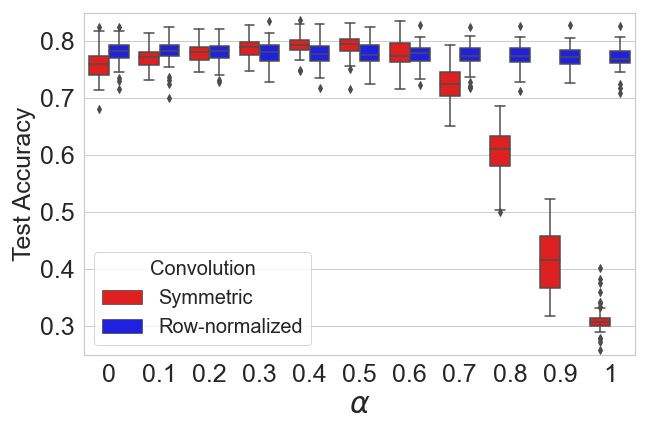}
    \caption{Accuracy as a function of $\alpha$ ($\beta=1$).}
     \end{subfigure}
\begin{subfigure}[t]{0.62\textwidth}
   \centering
    \includegraphics[width=\textwidth]{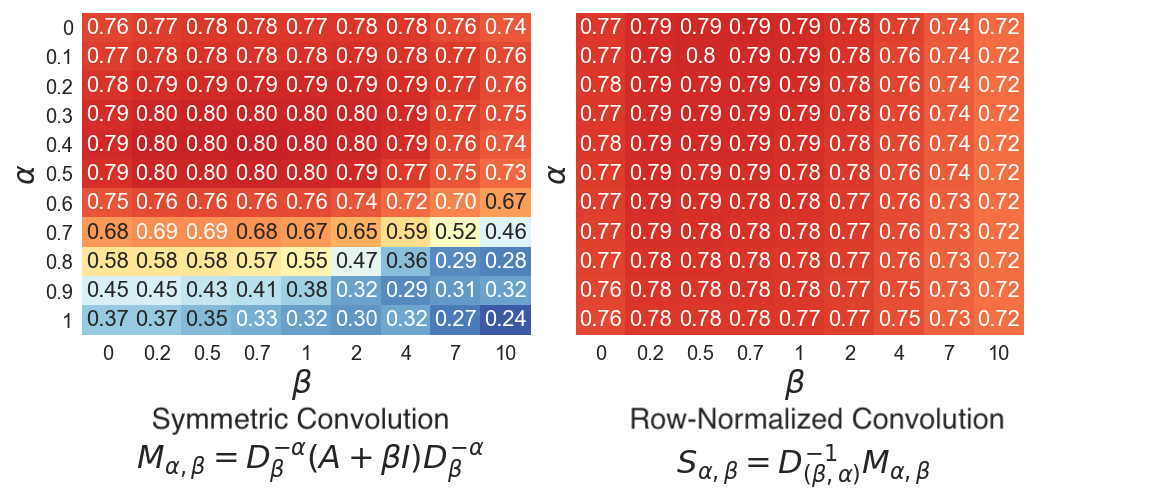}
    \caption{Cora dataset: test accuracy as a function of $\alpha$ and $\beta$}\label{fig:introb}     
    \end{subfigure}
    \caption{Results for Cora for our family of convolutions defined in Eq.\ref{eq:normalized} and Eq.\ref{eq:regularized} (50 independent experiments, selecting a random training set and test set). Note the strong dependency of the results on both $\alpha$ and $\beta$ for the normalized convolution.}
    \label{fig:intro}
\end{figure}

\begin{figure}\label{fig:intro}
     \centering
          \begin{subfigure}[t]{0.33\textwidth}
              \centering
    \includegraphics[width=1.0\textwidth, height=4cm]{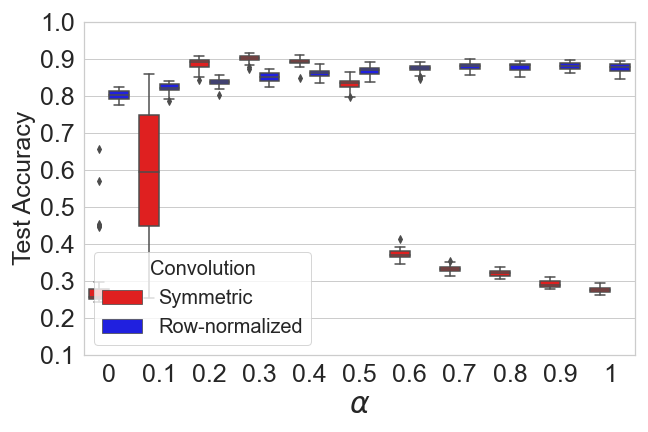}
    \caption{Amazon Photos ($h_e=0.83$)}
     \end{subfigure}
          \begin{subfigure}[t]{0.32\textwidth}
              \centering
    \includegraphics[width=1.0\textwidth, height=4cm]{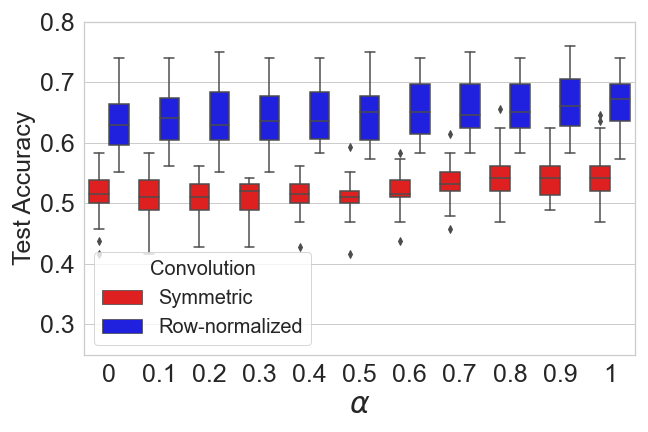}
    \caption{Cornell   ($h_e=0.30$)}
     \end{subfigure}
              \begin{subfigure}[t]{0.32\textwidth}
              \centering
    \includegraphics[width=1.0\textwidth, height=4cm]{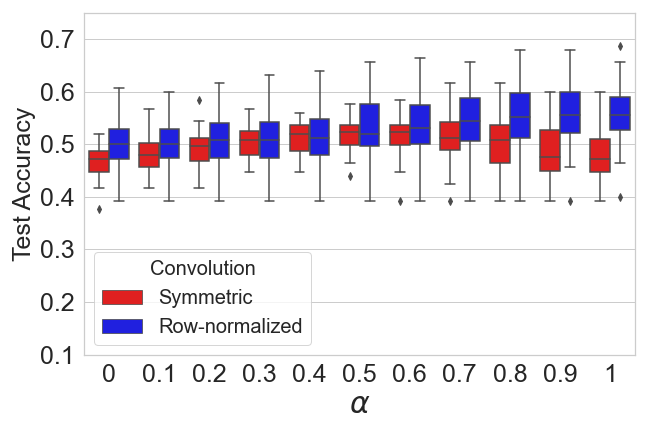}
    \caption{Wisconsin  ($h_e=0.21$)}
     \end{subfigure}
    \caption{Effect of $\alpha$ on the performance of the algorithm for our family of convolutions defined in Eq.\ref{eq:normalized} and Eq.\ref{eq:regularized} (30 independent experiments, with random training and test set). Here, $h_e$ denotes the edge homophily in the dataset (defined as the fraction of edges whose vertices share the same label) Note the strong dependency of the results on $\alpha$. See Appendix \ref{appendix:experiments} for further details and results.}
    \label{fig:intro2}
\end{figure}

Focusing on the effect of $\alpha$ (and setting $\beta=1$), we try using our convolution operators with varying intensities of $\a$ (using a GCN architecture), and report the test accuracy and standard deviation over 30 random experiments. The results are presented in Table~\ref{tab:experiments} (see Appendix \ref{appendix:experiments} for the experiment parameters). We highlight three main points: (i) $\alpha$ significantly impacts the performance of the GCN architecture;and (ii) $\a$ is tunable and should be optimized -- in fact, the default value $\alpha=0.5$ that is usually used in for the convolution operator is not necessarily the best. Choosing a value of $\alpha=0.1$ for the CoAuthor CS dataset for instance results in an average test accuracy of 93\%, a 4.37\% increase from that using the default value $\a=0.5$. Similarly, using $\alpha=0.3$ on Amazon increases the accuracy by 6.7\% compared to $\a=0.5$. Finally, we observe that (iii) the row-normalized operator is substantially more robust to the tuning of $\alpha$. These observations seem to hold across datasets, as shown in Appendix \ref{appendix:experiments}.  

\begin{table}[h]
\resizebox{\columnwidth}{!}{
\begin{tabular}{c|c|c|c|c|c|c|c|c|c|c|c|}
\cline{2-12}
\rowcolor[HTML]{C0C0C0} 
\cellcolor[HTML]{FFFFFF}\textbf{} & \textbf{Alpha} & \cellcolor[HTML]{C0C0C0} & \cellcolor[HTML]{C0C0C0} & \cellcolor[HTML]{C0C0C0} & \cellcolor[HTML]{C0C0C0} & \cellcolor[HTML]{C0C0C0} & \cellcolor[HTML]{C0C0C0} & \cellcolor[HTML]{C0C0C0} & \cellcolor[HTML]{C0C0C0} & \cellcolor[HTML]{C0C0C0} & \cellcolor[HTML]{C0C0C0} \\ \cline{1-2}
\rowcolor[HTML]{C0C0C0} 
\multicolumn{1}{|c|}{\cellcolor[HTML]{C0C0C0}\textbf{Dataset}} & \textbf{Convolution Type} & \multirow{-2}{*}{\cellcolor[HTML]{C0C0C0}\textbf{0.1}} & \multirow{-2}{*}{\cellcolor[HTML]{C0C0C0}\textbf{0.2}} & \multirow{-2}{*}{\cellcolor[HTML]{C0C0C0}\textbf{0.3}} & \multirow{-2}{*}{\cellcolor[HTML]{C0C0C0}\textbf{0.4}} & \multirow{-2}{*}{\cellcolor[HTML]{C0C0C0}\textbf{0.5}} & \multirow{-2}{*}{\cellcolor[HTML]{C0C0C0}\textbf{0.6}} & \multirow{-2}{*}{\cellcolor[HTML]{C0C0C0}\textbf{0.7}} & \multirow{-2}{*}{\cellcolor[HTML]{C0C0C0}\textbf{0.8}} & \multirow{-2}{*}{\cellcolor[HTML]{C0C0C0}\textbf{0.9}} & \multirow{-2}{*}{\cellcolor[HTML]{C0C0C0}\textbf{1.0}} \\ \hline
\multicolumn{1}{|c|}{} & \cellcolor[HTML]{FFFFFF}\textbf{Symmetric} & \cellcolor[HTML]{FFFFFF}77.02±2.13 & \cellcolor[HTML]{FFFFFF}77.85±1.7 & \cellcolor[HTML]{FFFFFF}78.91±1.78 & \cellcolor[HTML]{C0C0C0}\textbf{79.37±1.78} & \cellcolor[HTML]{FFFFFF}79.23±2 & \cellcolor[HTML]{FFFFFF}77.63±2.53 & \cellcolor[HTML]{FFFFFF}72.71±3.43 & \cellcolor[HTML]{FFFFFF}60.67±4.09 & \cellcolor[HTML]{FFFFFF}41.59±5.71 & \cellcolor[HTML]{FFFFFF}31.19±2.59 \\ \cline{2-12} 
\multicolumn{1}{|c|}{\multirow{-2}{*}{\textbf{Cora}}} & \cellcolor[HTML]{EFEFEF}\textbf{Row-Normalized} & \cellcolor[HTML]{9B9B9B}{\color[HTML]{000000} \textbf{78.12±2.27}} & \cellcolor[HTML]{EFEFEF}78.08±2.01 & \cellcolor[HTML]{EFEFEF}77.97±2.16 & \cellcolor[HTML]{EFEFEF}77.83±2.1 & \cellcolor[HTML]{EFEFEF}77.81±2.07 & \cellcolor[HTML]{EFEFEF}77.73±1.91 & \cellcolor[HTML]{EFEFEF}77.3±2.24 & \cellcolor[HTML]{EFEFEF}77.33±2.21 & \cellcolor[HTML]{EFEFEF}77.24±2.01 & \cellcolor[HTML]{EFEFEF}77.1±2.2 \\ \hline
\multicolumn{1}{|c|}{} & \cellcolor[HTML]{FFFFFF}\textbf{Symmetric} & \cellcolor[HTML]{FFFFFF}75.41±2.51 & \cellcolor[HTML]{FFFFFF}76.09±2.38 & \cellcolor[HTML]{FFFFFF}76.63±2.51 & \cellcolor[HTML]{C0C0C0}\textbf{76.79±2.48} & \cellcolor[HTML]{FFFFFF}75.4±3.81 & \cellcolor[HTML]{FFFFFF}68.27±6.82 & \cellcolor[HTML]{FFFFFF}54.62±8.52 & \cellcolor[HTML]{FFFFFF}44.11±5.36 & \cellcolor[HTML]{FFFFFF}40.57±2.13 & \cellcolor[HTML]{FFFFFF}39.9±2.09 \\ \cline{2-12} 
\multicolumn{1}{|c|}{\multirow{-2}{*}{\textbf{PubMed}}} & \cellcolor[HTML]{EFEFEF}\textbf{Row-Normalized} & \cellcolor[HTML]{EFEFEF}74.42±3.82 & \cellcolor[HTML]{EFEFEF}74.62±3.76 & \cellcolor[HTML]{EFEFEF}74.74±3.65 & \cellcolor[HTML]{9B9B9B}\textbf{74.75±3.44} & \cellcolor[HTML]{EFEFEF}74.68±3.45 & \cellcolor[HTML]{EFEFEF}74.48±3.44 & \cellcolor[HTML]{EFEFEF}74.28±3.38 & \cellcolor[HTML]{EFEFEF}73.92±3.36 & \cellcolor[HTML]{EFEFEF}73.73±3.36 & \cellcolor[HTML]{EFEFEF}73.42±3.44 \\ \hline
\multicolumn{1}{|c|}{} & \cellcolor[HTML]{FFFFFF}\textbf{Symmetric} & \cellcolor[HTML]{FFFFFF}65.06±2.74 & \cellcolor[HTML]{FFFFFF}65.5±2.57 & \cellcolor[HTML]{FFFFFF}66.51±2.49 & \cellcolor[HTML]{C0C0C0}\textbf{67.59±2.59} & \cellcolor[HTML]{FFFFFF}67.45±2.56 & \cellcolor[HTML]{FFFFFF}67.32±3.04 & \cellcolor[HTML]{FFFFFF}65.53±3.58 & \cellcolor[HTML]{FFFFFF}60.77±4.17 & \cellcolor[HTML]{FFFFFF}51.41±4.86 & \cellcolor[HTML]{FFFFFF}38.2±5.07 \\ \cline{2-12} 
\multicolumn{1}{|c|}{\multirow{-2}{*}{\textbf{Citeseer}}} & \cellcolor[HTML]{EFEFEF}{\color[HTML]{000000} \textbf{Row-Normalized}} & \cellcolor[HTML]{EFEFEF}{\color[HTML]{000000} 66.69±2.87} & \cellcolor[HTML]{EFEFEF}{\color[HTML]{000000} 66.75±2.99} & \cellcolor[HTML]{EFEFEF}{\color[HTML]{000000} 66.79±2.69} & \cellcolor[HTML]{EFEFEF}{\color[HTML]{000000} 66.99±3.07} & \cellcolor[HTML]{EFEFEF}{\color[HTML]{000000} 66.74±2.71} & \cellcolor[HTML]{EFEFEF}{\color[HTML]{000000} 66.91±2.84} & \cellcolor[HTML]{9B9B9B}{\color[HTML]{000000} \textbf{67.02±3.03}} & \cellcolor[HTML]{EFEFEF}{\color[HTML]{000000} 66.98±2.92} & \cellcolor[HTML]{EFEFEF}{\color[HTML]{000000} 66.93±2.92} & \cellcolor[HTML]{EFEFEF}{\color[HTML]{000000} 66.96±2.97} \\ \hline
\multicolumn{1}{|c|}{} & \cellcolor[HTML]{FFFFFF}\textbf{Symmetric} & \cellcolor[HTML]{C0C0C0}{\color[HTML]{000000} \textbf{93.02±0.27}} & \cellcolor[HTML]{FFFFFF}{\color[HTML]{000000} 92.93±0.25} & \cellcolor[HTML]{FFFFFF}{\color[HTML]{000000} 92.9±0.21} & \cellcolor[HTML]{FFFFFF}{\color[HTML]{000000} 92.32±0.26} & \cellcolor[HTML]{FFFFFF}{\color[HTML]{000000} 88.66±0.39} & \cellcolor[HTML]{FFFFFF}{\color[HTML]{000000} 77.88±0.89} & \cellcolor[HTML]{FFFFFF}{\color[HTML]{000000} 52.34±1.93} & \cellcolor[HTML]{FFFFFF}{\color[HTML]{000000} 24.91±0.48} & \cellcolor[HTML]{FFFFFF}{\color[HTML]{000000} 22.63±0.29} & \cellcolor[HTML]{FFFFFF}{\color[HTML]{000000} 22.63±0.29} \\ \cline{2-12} 
\multicolumn{1}{|c|}{\multirow{-2}{*}{\textbf{Coauthor CS}}} & \cellcolor[HTML]{EFEFEF}\textbf{Row-Normalized} & \cellcolor[HTML]{EFEFEF}{\color[HTML]{000000} 88.97±0.4} & \cellcolor[HTML]{EFEFEF}{\color[HTML]{000000} 89.4±0.43} & \cellcolor[HTML]{EFEFEF}{\color[HTML]{000000} 89.74±0.41} & \cellcolor[HTML]{EFEFEF}{\color[HTML]{000000} 90.01±0.45} & \cellcolor[HTML]{9B9B9B}{\color[HTML]{000000} \textbf{90.13±0.51}} & \cellcolor[HTML]{EFEFEF}{\color[HTML]{000000} 90.21±0.5} & \cellcolor[HTML]{EFEFEF}{\color[HTML]{000000} 90.29±0.5} & \cellcolor[HTML]{EFEFEF}{\color[HTML]{000000} 90.26±0.47} & \cellcolor[HTML]{EFEFEF}{\color[HTML]{000000} 90.3±0.44} & \cellcolor[HTML]{EFEFEF}{\color[HTML]{000000} 90.24±0.42} \\ \hline
\multicolumn{1}{|c|}{} & \cellcolor[HTML]{FFFFFF}\textbf{Symmetric} & \cellcolor[HTML]{FFFFFF}58.77±18.34 & \cellcolor[HTML]{FFFFFF}88.67±1.62 & \cellcolor[HTML]{C0C0C0}\textbf{90.09±1.02} & \cellcolor[HTML]{FFFFFF}89.18±1.07 & \cellcolor[HTML]{FFFFFF}83.23±1.64 & \cellcolor[HTML]{FFFFFF}37.12±1.4 & \cellcolor[HTML]{FFFFFF}33.17±0.93 & \cellcolor[HTML]{FFFFFF}32±0.91 & \cellcolor[HTML]{FFFFFF}29.05±0.94 & \cellcolor[HTML]{FFFFFF}27.47±0.86 \\ \cline{2-12} 
\multicolumn{1}{|c|}{\multirow{-2}{*}{\textbf{Amazon Photo}}} & \cellcolor[HTML]{EFEFEF}\textbf{Row-Normalized} & \cellcolor[HTML]{EFEFEF}82.09±1.34 & \cellcolor[HTML]{EFEFEF}83.75±1.13 & \cellcolor[HTML]{EFEFEF}84.97±1.34 & \cellcolor[HTML]{EFEFEF}86.15±1.26 & \cellcolor[HTML]{EFEFEF}86.75±1.44 & \cellcolor[HTML]{EFEFEF}87.33±1.18 & \cellcolor[HTML]{EFEFEF}88.01±1.11 & \cellcolor[HTML]{EFEFEF}87.9±1.2 & \cellcolor[HTML]{9B9B9B}{\color[HTML]{000000} \textbf{88.05±0.99}} & \cellcolor[HTML]{EFEFEF}87.46±1.3 \\ \hline
\multicolumn{1}{|c|}{} & \cellcolor[HTML]{FFFFFF}\textbf{Symmetric} & \cellcolor[HTML]{FFFFFF}27.06±0.61 & \cellcolor[HTML]{FFFFFF}27.43±0.64 & \cellcolor[HTML]{FFFFFF}27.74±0.61 & \cellcolor[HTML]{FFFFFF}28.27±0.54 & \cellcolor[HTML]{FFFFFF}28.58±0.57 & \cellcolor[HTML]{FFFFFF}28.75±0.52 & \cellcolor[HTML]{C0C0C0}\textbf{28.9±0.56} & \cellcolor[HTML]{FFFFFF}28.83±0.6 & \cellcolor[HTML]{FFFFFF}28.51±0.67 & \cellcolor[HTML]{FFFFFF}28.33±0.64 \\ \cline{2-12} 
\multicolumn{1}{|c|}{\multirow{-2}{*}{\textbf{Actor}}} & \cellcolor[HTML]{EFEFEF}\textbf{Row-Normalized} & \cellcolor[HTML]{EFEFEF}29.42±0.73 & \cellcolor[HTML]{EFEFEF}29.97±0.54 & \cellcolor[HTML]{EFEFEF}30.33±0.65 & \cellcolor[HTML]{EFEFEF}30.71±0.71 & \cellcolor[HTML]{EFEFEF}31±0.83 & \cellcolor[HTML]{EFEFEF}31.26±0.71 & \cellcolor[HTML]{EFEFEF}31.37±0.72 & \cellcolor[HTML]{EFEFEF}31.53±0.6 & \cellcolor[HTML]{EFEFEF}31.58±0.6 & \cellcolor[HTML]{9B9B9B}\textbf{31.63±0.6} \\ \hline
\multicolumn{1}{|c|}{} & \cellcolor[HTML]{FFFFFF}\textbf{Symmetric} & \cellcolor[HTML]{FFFFFF}51.04±4.08 & \cellcolor[HTML]{FFFFFF}51.01±3.25 & \cellcolor[HTML]{FFFFFF}51.04±2.91 & \cellcolor[HTML]{FFFFFF}51.28±3.08 & \cellcolor[HTML]{FFFFFF}51.25±3.12 & \cellcolor[HTML]{FFFFFF}52.29±3.06 & \cellcolor[HTML]{FFFFFF}53.65±3.21 & \cellcolor[HTML]{FFFFFF}54.38±4.17 & \cellcolor[HTML]{C0C0C0}\textbf{54.55±3.97} & \cellcolor[HTML]{FFFFFF}54.55±4.65 \\ \cline{2-12} 
\multicolumn{1}{|c|}{\multirow{-2}{*}{\textbf{Cornell}}} & \cellcolor[HTML]{EFEFEF}\textbf{Row-Normalized} & \cellcolor[HTML]{EFEFEF}63.78±4.39 & \cellcolor[HTML]{EFEFEF}64.13±4.8 & \cellcolor[HTML]{EFEFEF}63.96±4.53 & \cellcolor[HTML]{EFEFEF}64.65±4.51 & \cellcolor[HTML]{EFEFEF}64.83±4.74 & \cellcolor[HTML]{EFEFEF}65.52±4.65 & \cellcolor[HTML]{EFEFEF}65.87±4.72 & \cellcolor[HTML]{EFEFEF}65.94±4.6 & \cellcolor[HTML]{9B9B9B}\textbf{66.46±4.5} & \cellcolor[HTML]{EFEFEF}66.39±4.61 \\ \hline
\multicolumn{1}{|c|}{} & \cellcolor[HTML]{FFFFFF}\textbf{Symmetric} & \cellcolor[HTML]{FFFFFF}48.03±3.69 & \cellcolor[HTML]{FFFFFF}49.39±3.5 & \cellcolor[HTML]{FFFFFF}50.29±3.04 & \cellcolor[HTML]{FFFFFF}51.09±3.15 & \cellcolor[HTML]{9B9B9B}\textbf{51.79±3.1} & \cellcolor[HTML]{FFFFFF}51.6±3.65 & \cellcolor[HTML]{FFFFFF}50.8±4.59 & \cellcolor[HTML]{FFFFFF}49.92±5.04 & \cellcolor[HTML]{FFFFFF}49.07±5.16 & \cellcolor[HTML]{FFFFFF}48.21±4.75 \\ \cline{2-12} 
\multicolumn{1}{|c|}{\multirow{-2}{*}{\textbf{Wisconsin}}} & \cellcolor[HTML]{EFEFEF}\textbf{Row-Normalized} & \cellcolor[HTML]{EFEFEF}50.24±4.88 & \cellcolor[HTML]{EFEFEF}51.04±5.38 & \cellcolor[HTML]{EFEFEF}51.28±5.77 & \cellcolor[HTML]{EFEFEF}52.03±6.07 & \cellcolor[HTML]{EFEFEF}52.91±6.32 & \cellcolor[HTML]{EFEFEF}53.6±5.82 & \cellcolor[HTML]{EFEFEF}54.4±6.01 & \cellcolor[HTML]{EFEFEF}55.28±6.06 & \cellcolor[HTML]{EFEFEF}55.76±6.01 & \cellcolor[HTML]{9B9B9B}\textbf{56.19±5.84} \\ \hline
\end{tabular}
}
\caption{Results across 30 different experiments across datasets.}
\label{tab:experiments}
\end{table}

We emphasize here that the scope of this paper is not to suggest another convolution operator that would achieve state-of-the-art results. Rather, through this series of experiments, we hope to have convinced the reader that, empirically, the choice of convolution is important and can help gain up to almost 7\% accuracy on traditional GNN approaches, with no modifications to the architecture of the network whatsoever. Motivated by these observations, the  rest of this paper focuses on analysing these convolution operators' characteristics, and subsequent impact on the geometry properties of the embedding space.

\vspace{-0.3cm}
\section{The geometrical organization of GNN embeddings in latent space}\label{sec:geometry}

In this section, we analyse the effect of the convolution operator on the overall organization of the latent space. Our objective is to (a) characterize the implicit constraints that these operators put on the geometry (we will see that nodes with different topological characteristics are pushed to concentrate in different parts of the embedding space), and (b) identify the downstream consequences in terms of performance. This discussion will be driven by considerations on nodes' topological characteristics --- rather than spectral arguments.
\vspace{-0.3cm}



\subsection{Symmetric Convolutions} 

We begin our study of the ``absolute'' latent geometry of our embeddings with the family of symmetric convolution $\mc{M}_{\alpha, \beta}$ (see Equation~\ref{eq:normalized}). Contrary to many theoretical analysis, our study here is based on an analysis of the last (linear) layer of the Graph Neural Network (see Equation~\ref{eq:last_layer}), since it has the benefit of being a simple linear layer on the node embeddings (since introduction), denoted as:
$$ H^{(K)} = S\s( H^{(K-1)}) = \sum_{v \in \mc{N}(u) \cup\{u\}} \frac{A_{uv}}{(d_u +\beta)^{\alpha}(d_v +\beta)^{\alpha}} Z_{v\cdot}  $$
where $ \mc{N}(u)$ denotes the neighborhood of node $u$, $A_{uv}$ is the (possibly weighted) adjacency matrix, with diagonal equal to $\beta$, and $Z_{v\cdot}  = \s( H^{(K-1)}_{v\cdot})$. In the rest of the paper, for ease of notation, we adopt the notation $\tilde{\mc{N}}(u)= \mc{N}(u)\cup \{u\}$.

As a first step to study of the effect of the choice of convolution operator on the latent embedding space, we propose the following lemma.
\begin{lemma}\label{lemma:dis}
The last layer  of the GNN is a simple linear layer on inputs such that for any node $u$:
 \begin{align}
  ||S Z||_2 \leq { ||Z||_{2,\infty}}  \Big ( ( d_u + \b)^{1-2\alpha} -  \alpha  \frac{\bar{\Delta}_u}{(d_u + \beta)^{2\alpha}} + \frac{\alpha(\alpha+1)M}{2}  \frac{\overline{\Delta^2}_u}{(d_u + \beta)^{1 + 2\alpha}} \Big) \label{eq:ineq3}
\end{align}
where $\bar{\Delta}_u$ (respectively $\overline{\Delta^2}_u$) are the weighted averages of the degree differences (respectively, squared degree differences): 
$$ \bar{\Delta}_u = \frac{ \sum_{v \in \tilde{\mc{N}}(u)} A_{uv} (d_v -d_u)} {d_u + \beta} \qquad \overline{\Delta^2}_u = \frac{ \sum_{v \in  \tilde{\mc{N}}(u)} A_{uv} (d_v -d_u)^2} {d_u + \beta}.$$  In this equation, $||Z||_{2, \infty} = \max_{v} ||Z_v||_2$, and $M =\frac{d_{\max} +\beta }{\beta +1})^{2}$, where $d_{\max}$ denote the maximal degree of the nodes in the network.
\end{lemma}

\begin{proof}
The proof is simple (see  Appendix \ref{appendix:extrinsic}), and relies on the triangle inequality coupled with a MacLaurin expansion of the function $d_v \to \frac{1}{(d_u +\beta)^{\alpha}(d_v +\beta)^{\alpha}} $ around $0$. 
\end{proof}

Note that this bound is not necessarily tight. In particular, the proof relies on an application of  the triangular inequality, along with H{\"o}lder's inequality to separate the convolution from the embeddings.   However, while potentially crude, this bound already allow us to shed more light into the behaviour of the embedding as a function of the parameters $\a$, $\b$, and their topology. In particular, this bound allows to highlight:
\begin{description}[noitemsep,leftmargin=0.5cm, topsep=0em]
\item[(a) The determining role of $\alpha$ in organizing the latent space.] The leading term in inequality \ref{eq:ineq3} is $d_u^{1-2\a}$, and offers a first explanation for the change of phase we have observed in some of our experiments in the previous section. For values of $\alpha<0.5$, this term is an increasing function of the degree: after even a single convolution, nodes with small degree have less leeway to spread, and will generally remain close to the origin; high-degree nodes, on the other hand, will be less limited, and able to spread to far greater radii. Conversely, if $\alpha>0.5$, the upper bound decreases for nodes with high degree --- forcing them to concentrate around the origin.  The parameter $\a$ therefore controls the "attraction" of nodes towards the origin of as a function of their degree.
\item[(b) The effect of $\beta$.] The coefficient $\beta$, on the other hand, acts as added mass to the degree $d_u$ and can be understood as the "strength" of the attraction: for $\a>0.5$, the attraction of high-degree nodes to the origin is an increasing function of $\b$. Conversely, for $\a<0.5$, the repulsion of the nodes from the origin is an increasing function of $\beta$. 
\item[(c) The influence of the surrounding topology] As previously exhibited, the node degree plays a defining role in "placing" the node in concentric circles around the origin. The bound also exhibits a dependency on the neighborhood topology through the terms $\bar{\Delta}_u$ and $\bar{\Delta^2}_u$. Therefore, the more topologically homogeneous the neighborhood, the lesser the impact of these extra terms. Moreover, as shown in the proof, $\beta$ also controls the variable $M$: the higher the $\beta$, the lesser the impact of the neighborhood.
\end{description}
\vspace{-0.1cm}

\xhdr{Consequences} The previous observations yield two main conclusions. First, the choice of the convolution vector drives the geometry of the embedding space: values of $\beta$ and $\alpha$ allow the embedding space to expand or contract around the origin. By induction, in a similar fashion to \cite{oono2019graph}, it is possible to show that consecutive convolutions followed by 1-Lipschitz (e.g. ReLU) non-linearities contract even further the embedding towards the origin.\\
The second consequence pertains the accuracy of the recovery. Since the last layer of GNN is a linear classifier, we can use known results from statistical theory about the influence of the different points on the performance of the algorithm.  In particular, in linear regression, it is known that high-leverage points (that is, points with ``extreme '' predictor values'') are more likely to be highly influential points \cite{weisberg2005applied} (the same follows for generalized linear models, with some nuances). As such, by preventing high-degree nodes (respectively low-degree) to take on extreme embedding value and concentrating them around the origin, the convolution operator is implicitly limits the amount of trust, or leverage, that these points may have. We summarize these observations in the following corollary.

\begin{corollary}[Effect of symmetric convolutions on node embeddings]
In networks with non-homogeneous degree distributions, the exponent $\alpha$ constrains the leverage associated to each of the embeddings as a function of their degree and local topology:
\begin{itemize}[noitemsep, topsep=-1em, leftmargin=0.5cm]
    \item{\bf For values of $\a>0.5$:} High-degree nodes are constrained to concentrate around the origin, allowing the performance of the algorithm be driven by more peripheral (low-degree) nodes.
      \item{\bf For high-values of $\a< 0.5$:} Low-degree nodes need lie closer to the origin than their high-degree counterparts, thereby allowing high-degree nodes to have higher leverage and potentially become more influential.
\end{itemize}
\end{corollary}


\xhdr{Experiments} To illustrate these bounds and check their validity, we perform a set of synthetic experiments. We generate a set of four cliques on 20 nodes (the ''hubs''). To each node in each clique, we add a link to a sparse Barabasi-Albert network on 10 nodes ('the periphery'), with parameter $m=1$, so that the average degree of the periphery is low ($\approx 1$). The peripheries are endowed with the same class label as their associated hub. Finally, we ensure that the network is connected by randomly connecting the hubs together (one new random link per hub). To generate node features, we take the first $k=4$ features to be the one-hot label vector, to which we concatenate 16 additional ``dummy features''(random Gaussians). We finally add Gaussian noise with scale $\s^2=4$ entrywise: the result is  a feature vector that is only weakly indicative of the class.  The trained embeddings are presented in Figure \ref{fig:exp_part2}. Note the lack of separability of the different classes based on their raw feature vectors, as captured by the PCA plot in the left column. As expected from our bounds, we observe an inversion of the geometry around the origin as $\a$ increases: high-degree nodes shift from the outskirts of the plot to being concentrated around the origin as $\a$ increases. We also refer the reader to Appendix~\ref{sec:app-results-degree}, in which we also verify these phenomena in standard datasets by providing PCA and UMAP plots of the corresponding node embeddings.

In order to test corollary~\ref{lemma:dis}, we modify this setting slightly, and now let the variance of the noise depend on the degree  $d_u$ of the node: $\s^2_u=e^{ 3 (-1.5 + \log(d_u))}$. This means that low degree nodes here have very small variance $\s<1$, while high-degree noise are extremely noisy $\s \approx 9$. Consequently, we expect that the geometries in which the high-degree nodes are placed on the outskirts (and have more leverage), and low-degree nodes are constrained to lie close to the origin will perform poorly. Conversely, we flip this scenario (we choose $\s^2_u=e^{ 3 (-1.5 + \log(d_{(n-u)})}$ where, if $d_u$ has rank $rk(u)$,  $d_{(n-rk(u))}$ is the degree of the  $n-rk(u)$ largest node), and expect the opposite phenomenon. The results are shown in Figure\ref{fig:exp_part2a}(a), and are well aligned with our expectations.


\begin{figure}
        \centering
    \includegraphics[width=\textwidth]{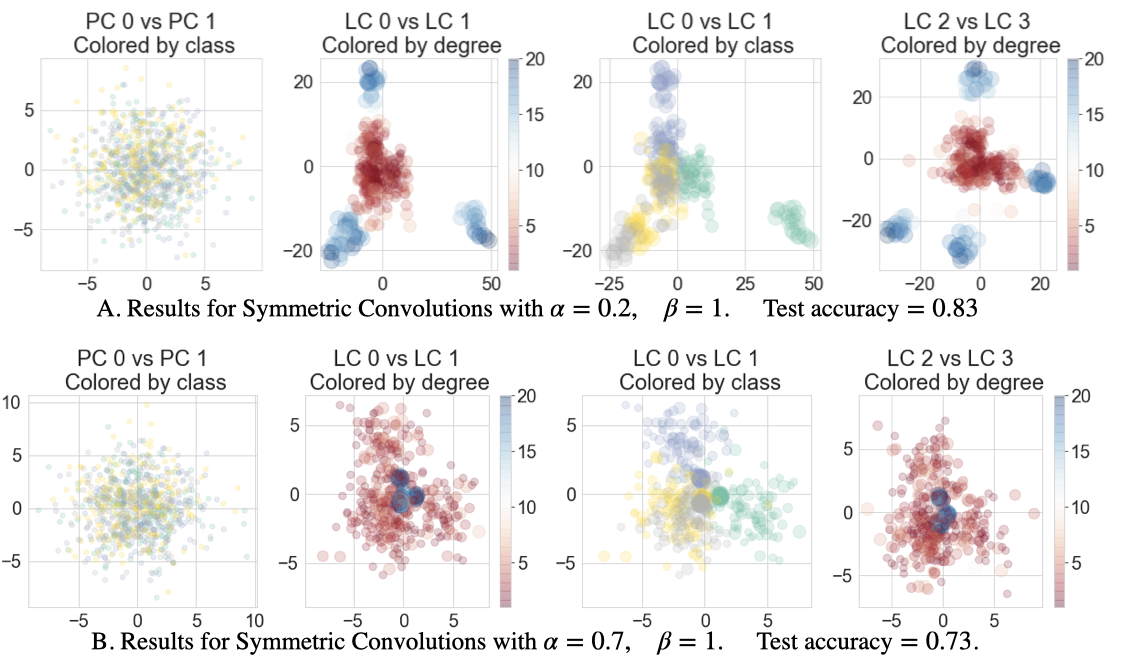}
    \caption{Symmetric Embeddings, plotted using the first two principal components (left), and the raw latent embeddings (or `latent components', shown in the right three plots on each row). Note the inversion: the high degree nodes migrate from the periphery of the latent space ($\a=0.2$) to the origin $(\a=0.7)$. See Appendix \ref{appendix:experiments} for the equivalent plot for row-normalized Embeddings.}
    \label{fig:exp_part2}
\end{figure}
\vspace{-0.2cm}

\subsection{Row-normalized Convolutions}
\vspace{-0.2cm}
Let us now turn to the case of row-normalized convolutions. In this case, the convolutions write as:
$$ SX =\frac{\frac{1}{(d_u+\b)^{\a}} \sum_{v \in \mc{N}(u) \cup\{u\}} \frac{1}{(d_v+\b)^{\a}}X_{v}}{\sum_{v \in \mc{N}(u)} \frac{1}{(d_u+\b)^{\a}}  \frac{1}{(d_v+\b)^{\a}} } =\frac{1}{\sum_{v \in \mc{N}(u)} \frac{1}{(d_v+\b)^{\a}} } \sum_{v \in \mc{N}(u) \cup\{u\}} \frac{1}{(d_v+\b)^{\a}}X_{v}$$
The embedding is thus no longer a function of the degree of the node. Rather, it lies within the convex hull of its neighbours, whose contributions are inversely proportional to their degrees. This effect can be strengthened through different choices of $\a$: the decay of a neighbor's contribution is an increasing function of $\a$. The latter can be compared to an form of attention that effectively filters out nodes with high-degree: here the discounting procedure is not learned, but imposed ahead of time. Appendix \ref{sec:app-additional} and Figure\ref{fig:exp_part2a}(b) show the result of the same experiments than in the last subsection. As in the previous part, the results are less dependent on the value of $\a$.

\begin{figure}
\begin{subfigure}[b]{0.49\textwidth}
\centering
\includegraphics[width=\textwidth]{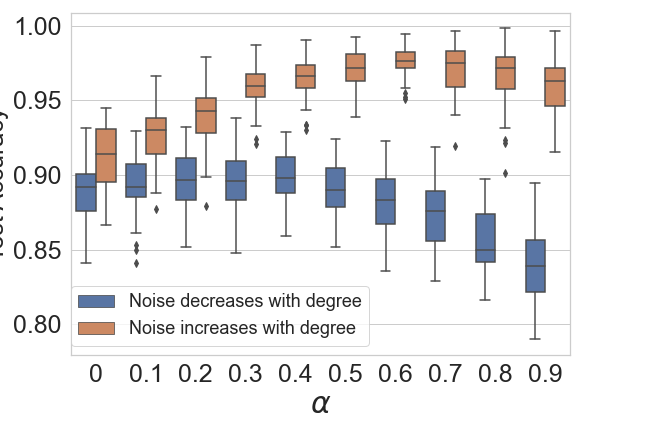}
\caption{Results for symmetric convolutions.}
    \label{fig:exp_part2a}
\end{subfigure}
\begin{subfigure}[b]{0.49\textwidth}
\centering
\includegraphics[width=\textwidth]{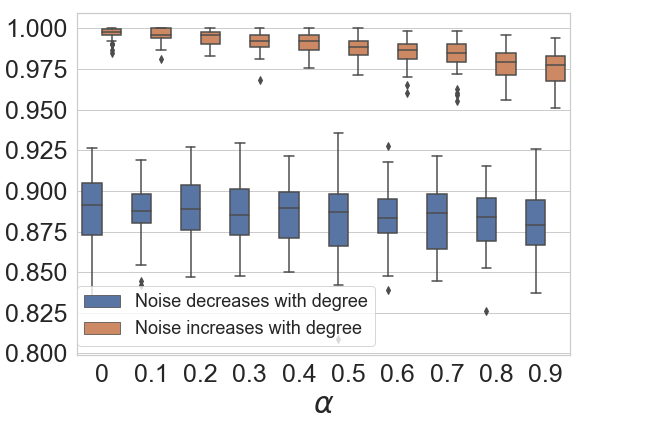}
\caption{Results for row-normalized convolutions.}
    \label{fig:exp_part2a}
\end{subfigure}
\caption{Results for 50 independent trials of our second experiment using symmetric (left) and row-normalized (right) convolutions. Note that the accuracy of the symmetric convolutions increases  $\a$ increases when the noise is higher in high-degree nodes, and decreases in the opposite scenario.}
\end{figure}

\vspace{-0.3cm}
\section{The inherent geometry of node embeddings}\label{sec:geometry-intrinsic}
\vspace{-0.3cm}

We now turn to the analysis of the evolution of the relative distances between embedding points.

 \xhdr{Embedding structural features} The  first point that we would like to raise in this section is that the two families of convolution operators differ radically in their encoding of topological features.o illustrate this, consider a simple two-layer GCN such as the one suggested by Kipf et al \cite{kipf2016semi}.  In this case, for the directions in which the term is positive,  the embedding $H$ (i.e., the transformed features that are being fed into the last linear layer) can be written as:
 \vspace{-0.35cm}
\begin{align}\label{eq:emb}
H_{u\cdot}= \sum_{k=1}^d \sum_{\substack{
        v \in \tilde{\mc{N}}(u)  \\
        (S XW + b)_{vk} \geq 0
    }} \big( S_{uv}(S XW)_{vk}W^{(2)}_{k\cdot} + S_{uv}b_{vk}W^{(2)}_{k\cdot} \big)\end{align} 
    The embedding is thus the sum of two components: a function of a (subset of) neighboring node features and a term that has the potential to encode local topology. To see why this is the case, consider a scenario where nodes in $\mc{N}(u)$ are all such that $ (S XW + b)_{vk} \geq 0$ for all $k$ or $ (S XW + b)_{vk} < 0$ for all $k$. Denote $\tilde{A}(u) = \{ v \in \tilde{N}(u):\quad (S XW + b)_{vk} \geq 0 \quad \text{for all } k \}$. In this case, Equation \ref{eq:emb} becomes:
$H_{u\cdot}= \sum_{
        v \in {\tilde{A}(u)}} \big( S_{uv}(S X\tilde{W})_{v\cdot} + S_{uv}\tilde{b}_{v\cdot} \big)$, with $\tilde{b} = bW^{(2)}$ and $\tilde{W}=WW^{(2)}$.
 
Using this notation, it becomes clear that:
\begin{itemize}[noitemsep, topsep=-0.8em, leftmargin=0.5cm]
    \item {\bf For row-normalized embeddings}, for two nodes to be close, it is sufficient for their neighborhood to have similar features. the distance between two features can be related to an ANOVA test between two neighborhood, and is devoid of topological information:
    $$ ||H_{u\cdot} - H_{u'\cdot}||^2 = \sum_{j=1}^p \Big(  \overline{(S X\tilde{W})_{\cdot j}}_{\tilde{A}(u)} -  \overline{(S X\tilde{W})_{\cdot j}}_{\tilde{A}(u')} \Big)^2,$$
    where $\overline{X}_{\tilde{A}(u)}$ denote the mean of $X$ over the set $\tilde{A}(u)$. Consequently, these  embeddings are appropriate when presumably, topology does not carry information relevant to the prediction.
       \item {\bf For symmetric embeddings}, on the other hand, the embedding depends strongly on the topology:
\[
           H_{uk}= \sum_{
        v \in {\tilde{A}(u)}} \Big(\sum_{
        w \in \mc{N}(v)} \frac{(X\tilde{W})_{wk} }{(d_u + \beta)^{\a}(d_v + \beta)^{2\a} (d_w + \beta)^{\a}}\Big) + \sum_{
        v \in {\tilde{A}(u)}} \frac{ \tilde{b}_{k}}{(d_u + \beta)^{\a}(d_v + \beta)^{\a}}
\]
In this case, note the existence of a bias term that is a function of both the source node's degree and that of its neighbours: this bias term can be seen as an offset that differentiates between high-degree nodes and low-degree nodes.
\end{itemize}

 Let us try and formalize this more precisely through two toy examples that allow us to control for feature similarity and node similarity separately.

\xhdr{Toy example 1: Identical Topologies, Different features} 
Consider two nodes $u$ and $v$ have structurally similar neighborhoods (i.e. there exists a mapping $\phi$ that transforms each node in the neighborhood of $v$ into its corresponding one in the neighborhood of $u$, see Figure~\ref{fig:str_eq}), but whose feature vectors are different. Mathematically, we write:
\begin{align*}
\forall j \in N(v), \quad  X_{j\cdot} = X_{\phi(j)\cdot} +\epsilon, \qquad \e_{jk} \overset{\text{i.i.d}}{\sim}N(0,\s^2).  \end{align*}
\begin{figure}
     \centering
\includegraphics[width=\textwidth]{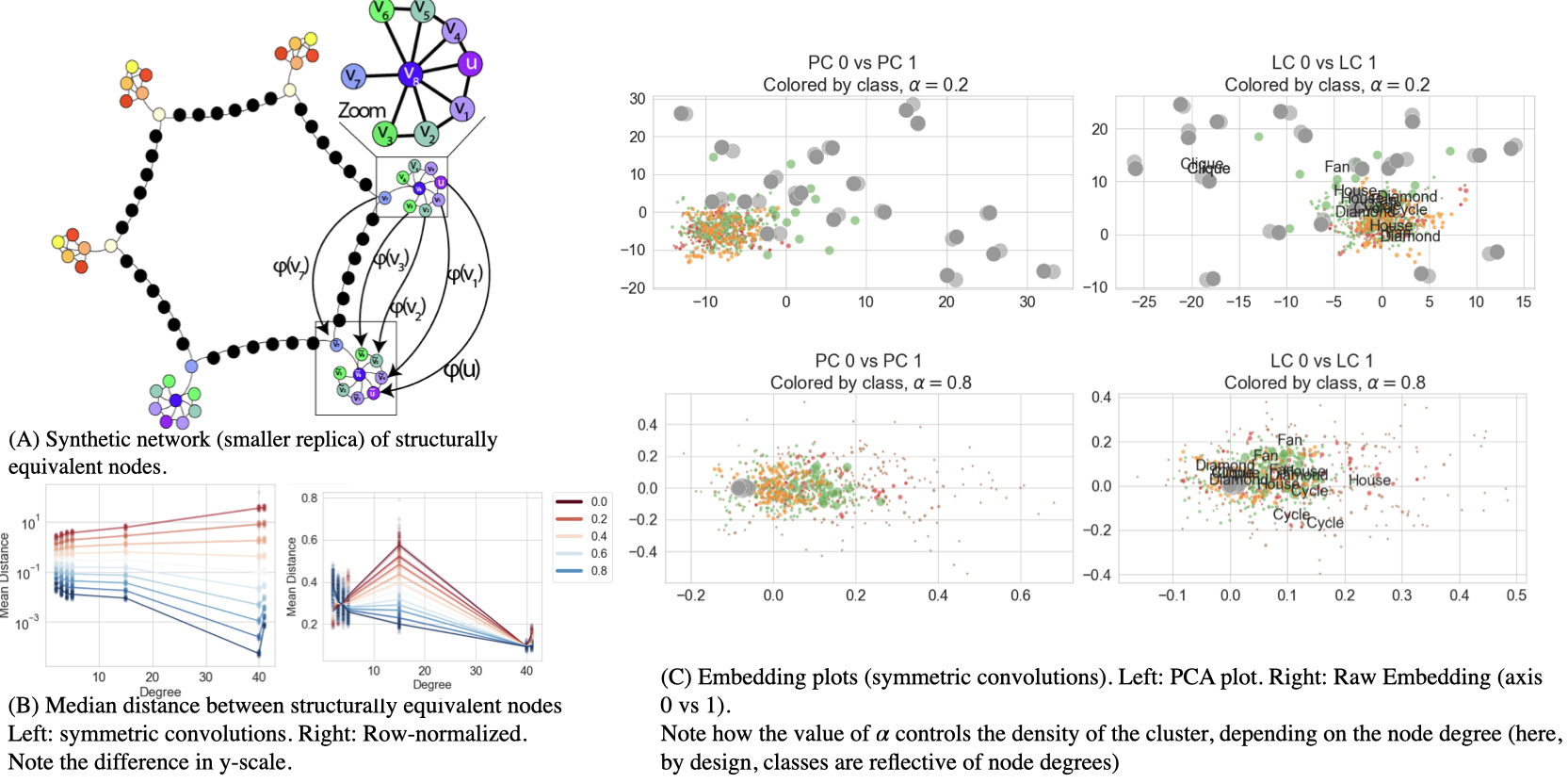}
\caption{Results for our Structural Equivalents experiment}\label{fig:str_eq}
\end{figure}
\vspace{-0.3cm}

\begin{lemma} \label{lemma:inh}
For symmetric convolution, with probability at least $1-\delta$, with $M$ as in \ref{lemma:dis}, we have: 
\begin{align*}
||H_u -H_{u'}||^2  \leq \mu +  2 \sqrt{2}\s||W||_2(d_u +\beta)^{1-2\a} \sqrt{1 +  2\a |\overline{\Delta}_u |({d_u + \beta})^{-1} +  \a(2\a+1) M} \log(1/\delta) \end{align*}
  where $\mu = \sigma^2||W||^2 \Big( (d_u +\beta)^{2-4\a} +  2\a |\overline{\Delta}_u |({d_u + \beta})^{-4\alpha} +  \a(2\a+1) M \overline{\Delta^2}_u({d_u + \beta})^{-1-4\alpha}\Big).$
Conversely, for row-normalized embeddings:
\vspace{-0.2cm}
\begin{align*}
||H_u -H_{u'}||^2 
  & \leq \mu +  2 \sqrt{2}\s||W||_2\sqrt{\sum_{v \in \tilde{\mathcal{N}}(u)} \frac{1}{(d_v +\beta)^{2\a}}} \log(1/\delta),  \quad  \mu = \frac{\sigma^2||W||^2}{\sum_{v \in \tilde{\mathcal{N}}(u)} (d_v +\beta)^{-2\a}} \frac{1}{1+\beta}  \end{align*}
\end{lemma}
This highlights fundamental differences in dealing with the embeddings: symmetric embeddings will shrink distances between nodes depending on their topology, so that the cluster density is a function of the node degree. The impact of the topology is expected to be more marked at the extremities of the spectrum of $\a$. The proof is in Appendix \ref{appendix:intrinsic}.

\xhdr{Toy example 2: Identical Features, structurally different neighbourhoods} 
Conversely, $u$ and $v$ have radically different neighborhoods from a topological perspective, but have similar features:
\[\forall j \in \tilde{\mc{N}}(u)\cup \tilde{\mc{N}}(u'), X_j = \bar{X} +\epsilon \]
In this case, as stated earlier in this section, for row-normalized embeddings, the difference will be 0: the embeddings are therefore more sensitive to the node feature values that the symmetric convolution.
Conversely, for symmetric convolutions,  we can also show (see Appendix \ref{appendix:intrinsic}) that the leading term for the difference is of the order of  $(d_u+\beta)^{1-2\a} - (d_{u'}+\beta)^{1-2\a}$ --- and is therefore a decreasing function of $\a$.

\xhdr{Experiments} We illustrate these results by running a set of final experiments. We generate a structurally equivalent networks (see smaller replica in Figure \ref{fig:str_eq}a), and evaluate the distance between untrained embeddings (using a 2-layer GCN architecture). The mean distance over 100 experiments is presented in Figure \ref{fig:str_eq}b, and a visualization of the latent space for symmetric convolutions is shown in Figure \ref{fig:str_eq}c. As expected, the density of the cluster of structurally equivalent high-degree nodes (cliques on 40 nodes), in grey varies as a function of $\a$. In Appendix\ref{sec:app-results-distance}, we also provide visualization of the interplay between node features and topologies on a subset of real datasets, using the Gromov Wasserstein distance as a way of measuring the distance of the embedding space with the original characteristics of the dataset (features and adjacency matrix),

\vspace{-0.3cm}
\section{Conclusion}
In conclusion, in this paper, we have shown that the choice of the convolution operator has fundamental consequences on the geometry of the embedding space: symmetric convolutions are generally more sensitive to the topology, and encode it in the embedding,. In that case, the choice of $\a$ amounts to selecting ``who to trust'': high-values of $\a$ push high-degree nodes towards the origin, thereby limiting their leverage. Conversely, row-normalized are more limited in the amount of topological information that they carry, and convolutions are more robust to the choice of $\a$ --- this is probably a better choice when the data is assumed to be sampled from a manifold (e.g. point cloud data).
Our analysis --- which we hope to be insightful --- has room for further improvement. Our reasoning relies on upper bounds which, while providing intuition, are not extremely tight, and could be complemented with lower bounds to fully characterize the behavior of the geometry. All experiments resort to using GCN types of architectures. However, we believe that the intuition and guidelines that we derived from this analysis will nonetheless hold for other types of architectures.

\bibliographystyle{plain}

\setcounter{tocdepth}{10}.

\appendix
\section{Comparison of GNN operators}\label{appendix:gnn}
\begin{table}[H]
\centering
 \begin{adjustbox}{max width=0.8\textwidth}
\begin{tabular}{|c|c|}
\hline
\textbf{Method} &
  \textbf{Operator} \\ \hline \hline
GCNConv \cite{kipf2016semi} &
  \begin{tabular}[c]{@{}c@{}}$ x_i' = \hat{D}^{-1/2} \hat{A} \hat{D}^{-1/2} XW$\\ $ x_i' =  \sum_{j \in \mathcal{N}(i)} \frac{e_{ij}}{\sqrt{\hat{d}_i\hat{d}_j}}x_j $\end{tabular} \\ \hline
ChebConv \cite{defferrard2016convolutional} &
  \begin{tabular}[c]{@{}c@{}}$X =  X W_1 + \hat{L}X W_2 + (2\hat{L}^2X - X)W_3$\\ with $\hat{L} = \frac{2}{\lambda_{\max}}L-I$\end{tabular} \\ \hline
SAGEConv \cite{hamilton2017inductive}&
  \begin{tabular}[c]{@{}c@{}}$x = W_1x_i + W_2 \bar{X}_{j \in \mathcal{N}(i)} $\\ with \\ $ \bar{X}_{j \in \mathcal{N}(i)} =\frac{\sum_{j \in \mathcal{N}(i)}x_j}{d_i}$\end{tabular} \\ \hline
GraphConv\cite{morris2019weisfeiler} &
  $x_i' = W_1 x_i + W_2\sum_{j \in \mathcal{N}(i)} e_{ij} x_j$ \\ \hline
GatedGraphConv \cite{li2015gated} &
  \begin{tabular}[c]{@{}c@{}}$h_i^{(0)} = x_i ||0$\\ $m_i^{(l+1)} = \sum_{j \in \mathcal{N}(i)} e_{j,i} Wh_j^{(l)}$\\ $h_i^{(l+1)} = GRU(m_i^{(l+1)}, h_i^{(l)} $\end{tabular} \\ \hline
ResGatedGraphConv \cite{bresson2017residual} &
  \begin{tabular}[c]{@{}c@{}}$x_i' = W_1x_i + \sum_{j \in \mathcal{N}(i)} \eta_{ij} \circ W_2x_j$\\ with\\ $\eta_{ij} =\sigma(W_3x_i + W_4x_j) $\end{tabular} \\ \hline
GAT \cite{velivckovic2017graph}/GATv2Conv \cite{brody2021attentive}&
  \begin{tabular}[c]{@{}c@{}}$x_i' = \alpha_{ii}Wx_i + \sum_{j\in \mathcal{N}(i)}\alpha_{ij}Wx_j $\\ with \\ $\alpha_{ij} =\frac{\exp\{ \text{LeakyReLU}(a^T[\Theta x_i || \Theta x_j]\}}{\sum_{k\in \mathcal{N}(i) \cup \{i\} } \exp\{ \text{LeakyReLU}(a^T[\Theta x_i || \Theta x_j]\}}$ \end{tabular} \\ \hline
AGNN \cite{thekumparampil2018attention}&
  \begin{tabular}[c]{@{}c@{}}$X' =PX$\\ $P_{ij} = \frac{ \exp\{ \beta \dot \cos(x_i,x_j) \}}{\sum_{k \in \mathcal{N}(i) \cup \{i\} }\exp\{ \beta \dot \cos(x_i,x_k) \} }$\end{tabular} \\ \hline
Transformer Conv \cite{shi2020masked}&
  \begin{tabular}[c]{@{}c@{}}$x_i' = W_1 x_i +\sum_{j \in \mathcal{N}(i)} \alpha_{ij} W_2 x_j$\\ $\alpha_{ij} =\text{Softmax} \frac{(W_3 x_i)^T (W_4 x_j)}{\sqrt{d}} $\end{tabular} \\ \hline
TAGConv \cite{du2017topology}&
  $ X' = \sum_{k=0}^K (D^{-1/2} AD^{-1/2})^k XW_k$ \\ \hline
GINConv \cite{xu2018powerful}&
  $X'  = h_{\theta}\Big( (\mathbf{A} + (1+\epsilon) \mathbf{I} )X \Big))$ \\ \hline
GINEConv \cite{hu2019strategies}&
  $x'_i  = h_{\theta}\Big( ( (1+\epsilon) x_i +\sum_{j \in \mathcal{N}(i)} \text{ReLU}(x_j +e_{ij}) \Big))$ \\ \hline
ARMAConv \cite{bianchi2021graph} &
  $X' =\frac{1}{K} \sum_{k=1}^K X_k^{(T)}$ \\ \hline
SGCConv \cite{wu2019simplifying} &
  $X' = (\hat{D}^{-1/2} \hat{A}\hat{D}^{-1/2} )^KXW$ \\ \hline
\end{tabular}
\end{adjustbox}
\caption{Comparison of the different convolution operators}\label{tab:operators_all}
\end{table}

\section{Proofs of Section \ref{sec:geometry}}\label{appendix:extrinsic}
\begin{proof}[Lemma \ref{lemma:dis}]
As per section \ref{sec:geometry}, we study the last (linear) layer of the Graph Neural Network (see Equation~\ref{eq:last_layer}). This layer has indeed the advantage of being linear --- and thereby easy to analyse. In this setting, we simply want to show that the use of a spatial convolution imposes a special geometry on the embeddings. By induction, we expect this result to hold for the other layers as well (assuming that the non-linearity that is chosen is ReLU --- other choices of non-linearities (e.g. $\tanh$), while linearisable close to 0, would require a more careful proof).  

In this section, we analyse the embedding that is fed into the last linear layer, denoted as:
$$ H^{(K)} = S\s( H^{(K-1)}) = \sum_{v \in \mc{N}(u) \cup\{u\}} \frac{A_{uv}}{(d_u +\beta)^{\alpha}(d_v +\beta)^{\alpha}} Z_{v\cdot}  $$
where $ \mc{N}(u)$ denotes the neighbourhood of node $u$, $A_{uv}$ is the (potentially weighted) adjacency matrix, with diagonal equal to $\beta$, and $Z_{v\cdot}  = \s( H^{(K-1)}_{v\cdot})$.

Writing $\Delta_v = {d_v-d_u}$, note that $H^{(K)}$ can be rewritten as:
\begin{align}\label{eq:taylor}
     H^{(K)} & = \frac{1}{(d_u +\beta)^{2\alpha}}\sum_{v \in \mc{N}(u) \cup\{u\}} \frac{A_{uv}}{(1 + \frac{\Delta_v}{d_u +\beta})^{\alpha}} Z_{v\cdot} 
\end{align}

Note that $\frac{\Delta_v}{d_u+\beta} \geq -1$ as long as $d_v-d_u\geq -d_u -\beta  \implies  d_v \geq -\beta$, which holds necessarily, since $d_v\geq 1$.  Since the function $x \to  ( x+ 1)^{-\alpha}$ is infinitely differentiable for $x \in (-1, \infty)$, using the Maclaurin expansion of $( x+ 1)^{-\alpha}$ around $0$, we know that there exists $ \xi\in [\min(0, x),\max(0, x)]$ such that:
 \begin{align} \label{eq:holder} \frac{1}{(x+1)^{\a}}  =1  - \a x +  \frac{\alpha(\alpha+1)}{2} \frac{ x^2}{(\xi+1)^{\alpha+2}}   \end{align}
It is easy to check that if $\frac{\Delta_v}{d_u+\beta}\geq 0$, then $\frac{1}{(1+\xi)^{\alpha+2}}\leq 1$.  Conversely, if $d_v\leq d_u$, then $\frac{1}{(\xi+1)^{\alpha+2}}\frac{1}{(1+\frac{\Delta_v}{d_u+\beta})^{\alpha+2}} \leq \frac{1}{(1+\frac{1-d_u}{d_u+\beta})^{\alpha+2}}\leq \frac{1}{(\frac{\beta+1}{d_u+\beta})^{2+\alpha}}\leq (d_{\max} + \beta)^{2 + \alpha}= M$.

Equation~\ref{eq:taylor} thus becomes: 
 \begin{align}
&  ||S Z||_2 \leq \frac{1}{(d_u + \b)^{2\a} } \sum_{v \in \mathcal{N}(u)\cup\{u\}} \frac{A_{uv}}{( 1 +  \frac{\Delta_v}{d_u + \beta} )^{\a}}  ||Z_v||_{2} \notag\\
&\leq \frac{ || Z||_{2,\infty}}{(d_u + \b)^{2\a} } \sum_{v \in \mathcal{N}(u)\cup\{u\}} A_{uv} ( 1 -\alpha  \frac{\Delta_v}{d_u + \beta} + \frac{\alpha(\alpha+1)M}{2}  \frac{\Delta^2_v}{(d_u + \beta)^2}) \notag\\
&={ ||Z||_{2,\infty}}  \Big ( ( d_u + \b)^{1-2\alpha} -  \alpha  \frac{\bar{\Delta}_v}{(d_u + \beta)^{2\alpha}} + \frac{\alpha(\alpha+1)M}{2}  \frac{\overline{\Delta^2}_v}{(d_u + \beta)^{1 + 2\alpha}} \Big) \label{eq:ineq}
\end{align}
where $\bar{\Delta}_u$ (respectively $\bar{\Delta^2}_u$) are the weighted averages of the degree differences: $\bar{\Delta}_u = \frac{ \sum_{v \in \mathcal{N}(u)\cup\{u\}} A_{uv} \Delta_v} {d_u + \beta}$ (respectively, squared degree differences: $\bar{\Delta^2}_u = \frac{ \sum_{v \in \mathcal{N}(u)\cup\{u\}} A_{uv} \Delta_v^2} {d_u + \beta}$). In the previous equation, we have also introduced the notation $||Z||_{2, \infty} = \max_{v} ||Z_v||_2$.

\begin{remark}
Note that, when $\s=\text{ReLU}$, then, by positivity of the entries in $Z$, we have the following (entrywise) inequality:
 \begin{align}
& 0\leq S Z \leq \frac{1}{(d_u + \b)^{2\a} } \sum_{v \in \mathcal{N}(u)\cup\{u\}} \frac{A_{uv}}{( 1 +  \frac{\Delta_v}{d_u + \beta} )^{\a}}  Z_v \notag\\
&\leq{ ||Z||_{\infty}}  \Big ( ( d_u + \b)^{1-2\alpha} -  \alpha  \frac{\bar{\Delta}_v}{(d_u + \beta)^{2\alpha}} + \frac{\alpha(\alpha+1)M}{2}  \frac{\overline{\Delta^2}_v}{(d_u + \beta)^{1 + 2\alpha}} \Big) \label{eq:ineq2}
\end{align}
\end{remark}

\end{proof}

\section{Proofs of Section \ref{sec:geometry-intrinsic}}\label{appendix:intrinsic}
\subsection{Section \ref{sec:geometry-intrinsic}: Proof of the observations}

We begin by revisiting in greater details the observations  made in section \ref{sec:geometry-intrinsic}.
To see how the two families of spatial operators differ in the importance they attribute to topology and node feature information, consider a simple two-layer GCN such as  suggested by Kipf et al \cite{kipf2016semi}. In this setting, node embeddings can be written as:
 $H = S  \s(S XW + b),$ so that the output of the network is $Y = S  \s(S XW + b)W^{(2)}+b^{(2)}=HW^{(2)} + b^{(2)}$. We also choose the non-linearity $\s$ to be the ReLU function. In this case, for the directions in which the term is positive,  the embedding $H$ (ie, the transformed features that are being fed into the last linear layer) can be re-written as:
\begin{align}\label{eq:emb}
H_{u\cdot}= \sum_{k=1}^d \sum_{\substack{
        v \in \tilde{\mc{N}}(u)  \\
        (S XW + b)_{vk} \geq 0
    }} \big( S_{uv}(S XW)_{vk}W^{(2)}_{k\cdot} + S_{uv}b_{vk}W^{(2)}_{k\cdot} \big)\end{align} 
    The embedding is thus the sum of two components: a function of a (subset of) neighbouring feature vector and a term that has the potential to encode local topology. To see why this is the case, consider a scenario where nodes in $\mc{N}(u)$ are all such that $ (S XW + b)_{vk} \geq 0$ for all $k$ or $ (S XW + b)_{vk} < 0$ for all $k$. Denote $\tilde{A}(u) = \{ v \in \tilde{N}(u):\quad (S XW + b)_{vk} \geq 0 \quad \text{for all } k \}$. In this case, Equation \ref{eq:emb} becomes:
$H_{u\cdot}= \sum_{
        v \in {\tilde{A}(u)}} \big( S_{uv}(S X\tilde{W})_{v\cdot} + S_{uv}\tilde{b}_{v\cdot} \big)$, with $\tilde{b} = bW^{(2)}$ and $\tilde{W}=WW^{(2)}$.
Therefore, for symmetric convolutions, the term $ (\sum_{
        v \in {\tilde{A}(u)}} S_{uv})b  $ encodes information about the neighborhood (it is proportional to the number of terms in the sum $|{\tilde{A}(u)}|$.)
        Conversely, for row-symmetric convolutions, this term is identically equal to $b$, resulting in an embedding that is less sensitive to topology.

\subsection{Proof of lemma \ref{lemma:inh}: Symmetric Convolutions}

In this subsection, we prove the results stated in lemma \ref{lemma:inh} for symmetric convolutions. We remind the reader of the setting of lemma \ref{lemma:inh}: we consider two structurally equivalent neighbourhoods (meaning that there exists a mapping $\phi$ that transforms each node in the neighborhood of $v$ into its corresponding one in the neighborhood of $u$ --- see Figure~\ref{fig:str_eq}), but the feature vectors are different. Mathematically, we model this situation as:
$$ \forall j \in N(v), \quad  X_j = X_{u\phi(j)} +\epsilon $$
where $\epsilon$ is a vector with independent centered Gaussian entries with parameter $\sigma$.  The purpose of this subsection is to analyze the effect of the convolution on the relative distance between embeddings. 

Lemma \ref{lemma:inh} is re-written here, to make this appendix self-contained:

\xhdr{Lemma 4.1}
{\it 
For symmetric convolutions, with probability at least $1-\delta$, with $M$ as in \ref{lemma:dis}, we have: 
\begin{align*}
||H_u -H_{u'}||^2  \leq \mu +  2 \sqrt{2}\s||W||_2(d_u +\beta)^{1-2\a} \sqrt{1 +  2\a |\overline{\Delta}_u |({d_u + \beta})^{-1} +  \a(2\a+1) M} \log(1/\delta) \end{align*}
  where $\mu = \sigma^2||W||^2 \Big( (d_u +\beta)^{2-4\a} +  2\a |\overline{\Delta}_u |({d_u + \beta})^{-4\alpha} +  \a(2\a+1) M \overline{\Delta^2}_u({d_u + \beta})^{-1-4\alpha}\Big).$
Conversely, for row-symmetric embeddings:
\vspace{-0.2cm}
\begin{align*}
||H_u -H_{u'}||^2 
  & \leq \mu +  2 \sqrt{2}\s||W||_2\sqrt{\sum_{v \in \tilde{\m{N}}(u)} \frac{1}{(d_v +\beta)^{2\a}}} \log(1/\delta),  \quad  \mu = \frac{\sigma^2||W||^2}{\sum_{v \in \tilde{\m{N}}(u)} (d_v +\beta)^{-2\a}} \frac{1}{1+\beta}  \end{align*}

}

As previously stated, the purpose of this subsection is to analyze the effect of the convolution on the relative distance between embeddings. Consequently, we consider a simplified one-layer setting, with no non-linearities. We argue that this is indeed sufficient to characterize the effect of the convolution on the organization of the data, and we expect results for deeper networks to follow by induction, and to hold by 1-Lipschitzness of the ReLU activation for ReLU non-linear GNNs. 

\begin{proof}{\bf \hspace{0.05cm} of Lemma~\ref{lemma:inh}(symmetric convolutions).}
Therefore, in the simplified setting, the distance between the outputs of a GCN layer for nodes $u$ and $v$ can be written as:
\begin{equation}
    \begin{split}
H^{(k)}_u -H^{(k)}_{u'}= (SX_{u}-SX_{u'})W &=\sum_{v\in \mc{N}(u)\cup \{u\}} \frac{A_{uv}}{(d_u+\beta)^{\alpha} (d_v +\beta)^{\alpha}}  (X_{v} -X_{\phi(v)})W \\
&= \sum_{v\in \mc{N}(u)\cup \{u\}} \frac{A_{uv}}{(d_u+\beta)^{\alpha} (d_v +\beta)^{\alpha}} \epsilon_vW\\
        \end{split}
\end{equation}
Since $\epsilon_v \sim \mc{N}(0, \sigma^2),$ each entry of the vector $ H^{(k)}_u -H^{(k)}_{u'}$ is Gaussian:
\begin{align*}
  H^{(k)}_{uj} -H^{(k)}_{u'j}=  \sum_{v\in \mc{N}(u)\cup \{u\}} \frac{A_{uv}}{(d_u+\beta)^{\alpha} (d_v +\beta)^{\alpha}} \sum_{k=1}^d \e_{vk}W_{kj} \sim \mc{N}(0, \frac{\sigma^2||W_{\cdot j}||^2}{(d_u+\beta)^{2\alpha}}\sum_{v\in \mc{N}(u)\cup \{u\}} \frac{A_{uv}^2}{(d_v +\beta)^{2\alpha}})
\end{align*} 

The mean of $||H^{(k)}_u -H^{(k)}_{u'}||^2$ is simply given by:
\begin{align*} \mu = \E[||H^{(k)}_u -H^{(k)}_{u'}||^2] &= \frac{\sigma^2||W||^2}{(d_u+\beta)^{2\alpha}}\sum_{v\in \mc{N}(u)\cup \{u\}} \frac{A_{uv}^2}{(d_v +\beta)^{2\alpha}}\\
&\leq  \frac{\sigma^2||W||^2}{(d_u+\beta)^{4\alpha}}\sum_{v\in \mc{N}(u)\cup \{u\}}A_{uv}^2  (1 -  2\a \frac{\Delta_v}{d_u + \beta} +  \a(2\a+1) M \frac{\Delta_v^2}{(d_u + \beta)^2}) \\
\end{align*}
for some constant $M$, using the same Taylor expansion reasoning as for Lemma~\ref{lemma:dis}. Therefore, denoting as $\tilde{d}_u =\sum_{v \in \mc{N}(u)\cup\{u\}} A_{uv}^2$, we have:
\begin{align*} \mu  &\leq  \frac{\sigma^2||W||^2}{(d_u+\beta)^{4\alpha}} \Big( (\tilde{d}_u +\beta^2) -  2\a \overline{\Delta}_u \frac{\tilde{d}_u+\beta^2}{d_u + \beta} +  \a(2\a+1) M \overline{\Delta^2}_u  \frac{\tilde{d}_u + \beta^2}{(d_u + \beta)^2}) \\
 &\overset{(i)}{\leq}  \frac{\sigma^2||W||^2}{(d_u+\beta)^{4\alpha}} \Big( (d_u +\beta)^2 +  2\a |\overline{\Delta}_u| \frac{({d}_u+\beta)^2}{d_u + \beta} +  \a(2\a+1) M \overline{\Delta^2}_u) \\
& =  \sigma^2||W||^2 \Big( (d_u +\beta)^{2-4\a} +  2\a |\overline{\Delta}_u |({d_u + \beta})^{1-4\alpha} +  \a(2\a+1) M \overline{\Delta^2}_u({d_u + \beta})^{-4\alpha}\Big) \\
\end{align*}
where line (i) follows from the fact that, assuming the edge weights are less than 1, $A_{uv}^2 \leq A_{uv}$, implying that $\tilde{d}_u \leq d_u$. Note that in the case where $\b\leq 1$ (so that $\b^2\leq 1)$, this bound can be made even tighter by considering:
$$\mu \leq  \sigma^2||W||^2 \Big( (d_u +\beta)^{2-4\a} +  2\a |\overline{\Delta}_u |({d_u + \beta})^{-4\alpha} +  \a(2\a+1) M \overline{\Delta^2}_u({d_u + \beta})^{-1-4\alpha}\Big).$$

Let us now turn to the analysis of the concentration of this norm. By Gaussianity of each of its coordinate, the squared norm
$|| H^{(k)}_u -H^{(k)}_{u'}||^2 = \sum_{j=1}^p \big(\sum_{v\in \mc{N}(u)\cup \{u\}} \frac{A_{uv}}{(d_u+\beta)^{\alpha} (d_v +\beta)^{\alpha}} \e_{v\cdot} W_{\cdot j}\big)^2 $ is  sub-exponential.

To see this, note that since each of the $p$ coordinate of the vector $H^{(k)}_u -H^{(k)}_{u'}$ is Gaussian with variance $\tilde{\s}_j^2=\frac{\sigma^2||W_{\cdot j}||^2}{(d_u+\beta)^{2\alpha}}\sum_{v\in \mc{N}(u)\cup \{u\}} \frac{A_{uv}^2}{(d_v +\beta)^{2\alpha}})$, its square is sub-Exponential with parameter $(2\tilde{\s}_j^2 ,4\tilde{\s}_j^2)$ (\cite{wainwright2019high}), so the squared norm (ie the sum of the squared entries) is sub-Exponential with parameter:
$$(2\sum_{j=1}^p \tilde{\s}_j^2 , 4\max_j \tilde{\s}_j^2) = \Big( 2\frac{\sigma^2||W||^2}{(d_u+\beta)^{2\alpha}}\sum_{v\in \mc{N}(u)\cup \{u\}} \frac{A_{uv}^2}{(d_v +\beta)^{2\alpha}}, 4 \frac{\sigma^2||W||_{2, \infty}^2}{(d_u+\beta)^{2\alpha}}\sum_{v\in \mc{N}(u)\cup \{u\}} \frac{A_{uv}^2}{(d_v +\beta)^{2\alpha}} \Big).$$ By property of the sub-exponential tail, we know that:
\begin{align}
 \P[||H^{(k)}_u -H^{(k)}_{u'}||^2  - \mu \geq t] \leq \min(e^{-t^2/(4\sum_{j=1}^p \tilde{\s}_j^2}, e^{-t/(2\sqrt{2}\sqrt{\sum_{j=1}^p \tilde{\s}_j^2}} )
 \end{align}
 Therefore, with probability at least $1-\delta$, for any $\delta \in (0,1)$, we must have:
\begin{align}
&||H^{(k)}_u -H^{(k)}_{u'}||^2  \leq \mu +  2\sqrt{2}
 \sqrt{\frac{\sigma^2||W||^2}{(d_u+\beta)^{2\alpha}}\sum_{v\in \mc{N}(u)\cup \{u\}} \frac{A_{uv}^2}{(d_v +\beta)^{2\alpha}}}\log(1/\delta) \\
  & \leq \mu +  2 \sqrt{2}\s||W||_2(d_u +\beta)^{1-2\a} \sqrt{1 +  2\a |\overline{\Delta}_u |({d_u + \beta})^{-1} +  \a(2\a+1) M \overline{\Delta^2}_u ({d_u + \beta})^{-2} } \log(1/\delta)
  \end{align}

The concentration is thus a function of the node degree: the leading term is in $(d_u + \beta)^{1-2\alpha}$, and we observe again the existence of a critical threshold at $\alpha=0.5$.
\end{proof}

\subsection{Proof of lemma \ref{lemma:inh}: the case of row-normalized Convolutions}

We now turn to the proof of Lemma~\ref{lemma:inh} for row-normalized convolutions.

\begin{proof}{\bf  \hspace{0.05cm}  of Lemma~\ref{lemma:inh}(row-normalized convolutions).}
In the case of row-normalized convolutions, we have instead:
\begin{equation}
    \begin{split}
(SX_{u}-SX_{u'})W &= \sum_{v\in \mc{N}(u)\cup \{u\}}s_{uv} \epsilon_vW\\
        \end{split}
\end{equation}
where, as highlighted in section \ref{sec:geometry}, $s_{uv}$ is proportional to  $\frac{1}{(d_v+ \beta)^{\a}}$, but does not depend on $d_u$. In this case, following a similar reasoning to the previous subsection:

\begin{align*} \mu = \E[||H^{(k)}_u -H^{(k)}_{u'}||^2] &= \frac{\sigma^2||W||^2}{Z^2}\sum_{v\in \mc{N}(u)\cup \{u\}} s_{uv}^2 \quad \text{with } Z = \sum_{v\in \mc{N}(u)\cup \{u\}} s_{uv} \\
&\leq  \frac{\sigma^2||W||^2}{Z^2} \max_{v\in \mc{N}(u)\cup \{u\}}\{s_{uv}\} \quad \text{ by Holder's inequality} \\
&\leq  \frac{\sigma^2||W||^2}{Z^2} \frac{1}{(d_{\min} +\beta)^{\alpha}}\\
&\leq  \frac{\sigma^2||W||^2}{Z^2} \frac{1}{(1+\beta)^{\alpha}} 
\end{align*}
\end{proof}

\subsection{Toy example 2}

Conversely, $u$ and $u'$ have radically different neighborhoods from a topological perspective, but have similar features:
$$ \forall j \in \mc{N}(u) \cup \mc{N}(u') ,\quad X_j = \bar{X}_u $$

\begin{description}[noitemsep, leftmargin=0.5cm]
\item[In the symmetric case:]
\begin{align*}
&||(SX)_{u\cdot} - (SX)_{u'\cdot}||^2  \\
&= || \sum_{v\in \tilde{\mc{N}}(u)} \frac{A_{uv}}{(d_u+\beta)^{\alpha}  (d_v +\beta)^{\alpha}}  \bar{X} - \sum_{v'\in \tilde{\mc{N}}(u')} \frac{A_{u'v'}}{(d_{v'}+\beta)^{\alpha} (d_{u'} +\b)^{\alpha}} \bar{X} ||^2\\
&= \Big( \sum_{v\in \tilde{\mc{N}}(u)} \frac{A_{uv}}{(d_u+\beta)^{\alpha}  (d_v +\beta)^{\alpha}} - \sum_{v'\in \tilde{\mc{N}}(u')} \frac{A_{u'v'}}{(d_{v'}+\beta)^{\alpha} (d_{u'} +\b)^{\alpha}} \Big)^2||\bar{X}||^2\\ 
&= \Big( \sum_{v\in \tilde{\mc{N}}(u)} \frac{A_{uv}}{(d_u+\beta)^{2\alpha}  (1 + \frac{\Delta_v}{d_u +\beta})^{\alpha}} - \sum_{v'\in \tilde{\mc{N}}(u')} \frac{A_{u'v'}}{(d_{u'}+\beta)^{2\alpha} (1 +\frac{\Delta_v}{d_{u'} +\beta})^{\alpha}} \Big)^2||\bar{X}||^2\\ 
&=  \Big( (d_u + \beta)^{1-2\alpha} -\alpha \overline{\Delta}_u + \frac{\alpha(\alpha+1)}{2(d_u + \beta)^{2 + 2\alpha}}  \sum_{v\in \tilde{\mc{N}}(u)}  \frac{A_{uv}(d_v-d_u)^2}{((1-t_v) + t_v\frac{\Delta_v}{d_u +\beta} + \beta)^{2 + 2\alpha}} \\
&- (d_{u'} + \beta)^{1-2\alpha} +\alpha \overline{\Delta}_{u'} - \frac{\alpha(\alpha+1)}{2(d_{u'} + \beta)^{2 + 2\alpha}}  \sum_{v'\in \tilde{\mc{N}}(u')}  \frac{A_{u'v'}(d_v'-d_u')^2}{((1-t_{v'}) + t_{v'}\frac{\Delta_{v'}}{d_{u'} + \beta} + \beta)^{2 + 2\alpha}} \Big)^2||\bar{X}||^2
\end{align*}
where $t_v, t_{v'} \in [0,1]$.
In this case, note that:
\begin{itemize}
    \item When $\a=0$, this difference writes as:
    $ ||(SX)_{u\cdot} - (SX)_{u'\cdot}||^2   = d_u  -d_{u'} $, and is thus extremely sensitive to the degree of the nodes,
    \item When $\a=1$, the difference can be written as:
   \begin{align*}
||(SX)_{u\cdot} - (SX)_{u'\cdot}||^2    &=  \Big( \frac{1}{d_u + \beta} - \overline{\Delta}_u + \frac{1}{(d_u + \beta)^{4}}  \sum_{v\in \tilde{\mc{N}}(u)}  \frac{A_{uv}(d_v-d_u)^2}{((1-t_v) + t_v\frac{\Delta_{v}}{d_{u} + \beta} + \beta)^{4}} \\
&- \frac{1}{d_{u'} + \beta} + \overline{\Delta}_{u'} - \frac{1}{(d_{u'} + \beta)^{4}}  \sum_{v'\in \tilde{\mc{N}}(u')}  \frac{A_{u'v'}(d_v'-d_u')^2}{((1-t_{v'})+ t_{v'}\frac{\Delta_{v'}}{d_{u'} + \beta} + \beta)^{4}} \Big)^2||\bar{X}||^2
   \end{align*} 
   In this case, the leading terms are functions of the inverse of the node degrees $\frac{1}{d_{u} + \beta}- \frac{1}{d_{u'} + \beta} $ and the difference in local homogeneity of topology $\overline{\Delta}_u - \overline{\Delta}_{u'}$. Consequently, the distance is still sensitive to topological properties of the neighborhood.
   \item When $\alpha=0.5$: in this case, the distance writes as:
   \begin{align*}
||(SX)_{u\cdot} - (SX)_{u'\cdot}||^2 &=  \Big(  \frac{1}{2}\overline{\Delta}_{u'}-\frac{1}{2}\overline{\Delta}_u \\
&+ \frac{3}{8(d_u + \beta)^{3}}  \sum_{v\in \tilde{\mc{N}}(u)}  \frac{A_{uv}(d_v-d_u)^2}{((1-t_v) + t_v\frac{\Delta_v}{d_u +\beta} + \beta)^{3}}\\
&- \frac{3}{8(d_{u'} + \beta)^{3}}  \sum_{v'\in \tilde{\mc{N}}(u')}  \frac{A_{u'v'}(d_{v'}-d_{u'})^2}{((1-t_{v'}) + t_{v'}\frac{\Delta_{v'}}{d_{u'} +\beta} + \beta)^{3}} \Big)^2||\bar{X}||^2
\end{align*}
Consequently, this distance is less directly related to the degree of the node, and relies more on the topological traits of the neighborhood.
\end{itemize}

\item[In the regularised case:]
\begin{equation}
    \begin{split}
        &||(SXW_1 + b_1)_u-(SXW_1 + b_1)_{u'} ||^2 =  0
        \end{split}
\end{equation}
In this case, the distance is entirely driven by the features.
\end{description}


\section{Further results and experiments}\label{appendix:experiments}

We analyzed the impact of the choice of operator across $\alpha$ and $\beta$ using standard benchmark datasets. In particular, the node classification task has been performed. The performance turns out to be dependent on the choice of an operator as well as the inherent characteristics of the datasets. We use visualizations to further investigate the properties of each embedding space in relation to the choice of operator. 

\subsection{Dataset statistics}

\begin{table}[H]
\centering
\resizebox{\columnwidth}{!}{%
\begin{tabular}{@{}l|ccccccll@{}}
\toprule
Name &
  Node &
  Edge &
  \begin{tabular}[c]{@{}c@{}}Features\end{tabular} &
  Class &
  \begin{tabular}[c]{@{}c@{}}Avg.Degree\end{tabular} &
  \begin{tabular}[c]{@{}c@{}}Mean\\ Centrality\end{tabular} &
  \multicolumn{1}{c}{$h_{\text{edges}}$} &
  \multicolumn{1}{c}{$h_{\text{nodes}}$} \\ \midrule
Cora          & 2,708  & 10,556  & 1,433 & 7  & 3.90  & $1.65E-03$ & {\color[HTML]{000000} 0.81} & {\color[HTML]{000000} 0.83} \\
Pubmed        & 19,717 & 88,648  & 500   & 3  & 4.50  & $2.71E-04$ & {\color[HTML]{000000} 0.80} & {\color[HTML]{000000} 0.79} \\
Citeseer      & 3,327  & 9,104   & 3,703 & 6  & 2.74  & $1.02E-03$ & {\color[HTML]{000000} 0.74} & {\color[HTML]{000000} 0.71} \\
Coauthor CS   & 18,333 & 163,788 & 6,805 & 15 & 8.93  & $2.42E-04$ & {\color[HTML]{000000} 0.81} & {\color[HTML]{000000} 0.83} \\
Amazon Photos & 7,650  & 238,162 & 745   & 8  & 31.13 & $3.82E-04$ & {\color[HTML]{000000} 0.83} & {\color[HTML]{000000} 0.84} \\
Actor         & 7,600  & 30,019  & 932   & 5  & 3.95  & $3.18E-04$ & {\color[HTML]{000000} 0.22} & {\color[HTML]{000000} 0.21} \\
Cornell       & 183    & 280     & 1,703 & 5  & 1.53  & $1.07E-04$ & {\color[HTML]{000000} 0.31} & {\color[HTML]{000000} 0.21} \\
Wisconsin     & 251    & 515     & 1,703 & 5  & 2.05  & $2.42E-04$ & {\color[HTML]{000000} 0.20} & {\color[HTML]{000000} 0.13} \\ \cmidrule(r){1-9}
\end{tabular}
}
\caption{Statistics for datasets used for experiments. Node and edge homophily indices are calculated by the formula suggested in \cite{Pei2020GeomGCN}, \cite{Zhu2020BeyonHomophily} respectively.}
\label{tab:data-stats}
\end{table}

\xhdr{Datasets} We used eight datasets used for experiments, namely, Cora, Pubmed, Citeseer, and Amazon Photos, Coauthor CS, as well as three heterophilic graph datasets, namely, Actor, Cornell, Wisconsin. We use the processed version provided by PyTorch Geometric\cite{fey2019fast}. Detailed statistics for the datasets used in the experiments are shown in Table \ref{tab:data-stats}.

\textit{Citation networks}. Cora, Citeseer, and Pubmed are standard citation network benchmark datasets.\cite{Yang2019Revisiting} In these networks, nodes represent scientific publications, and edges denote citation links between publications. Node features are the bag-of-words representation of papers, and
node label is the academic topic of a paper.

\textit{Coauthor} In Coauthor CS\cite{Shchur2018Pitfalls} network, each node represents the author of the scientific publication, and edge shows whether any of the authors coauthored the paper. Node features are bag-of-word representations of these documents, and node labels denote the field of study. 

\textit{Amazon} In Amazon Photo\cite{Shchur2018Pitfalls} network,  nodes represent goods and edges show whether two goods are frequently bought together. Node features are bag-of-word representation of product reviews.

\textit{WebKB}. WebKB is a webpage dataset collected from computer science departments of various universities by Carnegie Mellon University. We use Cornell, and Wisconsin among them. Nodes represent web pages, and edges are hyperlinks between them. Node features are the bag-of-words representation of web pages. The web pages are manually classified into the five categories: student, project, course, staff, and faculty.

\textit{Cooccurrence network} Actor dataset is the actor-only induced subgraph of the film-director-actor-writer network\cite{Pei2020GeomGCN}. Each node corresponds to an actor, and the edge between two nodes denotes co-occurrence on the same Wikipedia page. Node features correspond to corresponding Wikipedia keywords. 

\subsection{Experiment setup}

\xhdr{Models} We use a two-layer GCN\cite{kipf2016semi} model with varying families of spatial convolution operator across the choice of $\alpha$ while keeping $\beta \in \{0,1\}$. For each experiment, we randomly split the data into training and test sets (using the default number of train and test points in Pytorch geometric). The number of training nodes used are specified in Table \ref{table:char}. Further training details including data split, number of experiments and learning rate are also summarized in Table \ref{table:char}. 

\begin{table}[H]
\centering
\resizebox{\columnwidth}{!}{%
\begin{tabular}{@{}c|rrrrrrr@{}}
\toprule
\textbf{Name} &
  \multicolumn{1}{c}{\textbf{\begin{tabular}[c]{@{}c@{}}num of   \\ training nodes\end{tabular}}} &
  \multicolumn{1}{c}{\textbf{lr}} &
  \multicolumn{1}{c}{\textbf{epoch}} &
  \multicolumn{1}{c}{\textbf{\begin{tabular}[c]{@{}c@{}}num of \\ experiments\end{tabular}}} &
  \multicolumn{1}{c}{\textbf{\begin{tabular}[c]{@{}c@{}}training time\\ per epoch(sec)\end{tabular}}} &
  \multicolumn{1}{c}{\textbf{\begin{tabular}[c]{@{}c@{}}dim of \\ hidden layer\end{tabular}}} &
  \multicolumn{1}{c}{\textbf{\begin{tabular}[c]{@{}c@{}}num of \\ GCN layers\end{tabular}}} \\ \midrule
Cora          & 140           & 0.02  & 200 & 50 & 9.58   & {\color[HTML]{000000} 32} & {\color[HTML]{000000} 2} \\
Pubmed        & 60            & 0.001 & 500 & 30 & 100.49 & {\color[HTML]{000000} 32} & {\color[HTML]{000000} 2} \\
Citeseer      & \textit{1694} & 0.05  & 500 & 30 & 31.73  & {\color[HTML]{000000} 32} & {\color[HTML]{000000} 2} \\
Coauthor CS   & 9194          & 0.05  & 200 & 30 & 150.68 & {\color[HTML]{000000} 32} & {\color[HTML]{000000} 2} \\
Amazon Photos & 3844          & 0.05  & 200 & 30 & 88.04  & {\color[HTML]{000000} 32} & {\color[HTML]{000000} 2} \\
Actor         & 3804          & 0.02  & 200 & 30 & 13.02  & {\color[HTML]{000000} 32} & {\color[HTML]{000000} 2} \\
Cornell       & 87            & 0.01  & 500 & 30 & 2.95   & {\color[HTML]{000000} 32} & {\color[HTML]{000000} 2} \\
Wisconsin     & 126           & 0.001 & 500 & 30 & 3.52   & {\color[HTML]{000000} 32} & {\color[HTML]{000000} 2} \\ \bottomrule
\end{tabular}%
}
\caption{Hyperparameters and training details for all datasets.}\label{table:char}
\end{table}

\xhdr{Hardware and Software Specifications}. Our models are implemented with Python 3.8.8, PyTorch Geometric
2.0.5 \cite{fey2019fast}, and PyTorch 1.10.0 \cite{pytorch2019}. We conduct experiments on a computer equipped with 2.3 GHz Quad-Core Intel Core i7 processor and Intel Iris Plus Graphics 1536 MB. 

\subsection{Experiment results}

In this section, we highlight the results of our experiments on the various datasets aforementioned. Additional plots are provided in the folder of supplementary materials associated with this paper.

\subsubsection{Node Classification}
First, we want to investigate the impact of the choice of operators on the node classification task. We observe that the performance of the node classification task varies by choice of $\alpha$. We fix $\beta = 1$ --- in other words, we add self-loops, consistently with the standard GCN architecture. In general, we observe that performance are highly dependent on the choice of $\alpha$ especially for the symmetrized operator, but the performance of node classification task of row-normalized operator is relatively robust to the choice of $\alpha$

\begin{figure}[H]
    \centering
        \begin{subfigure}{1\textwidth}
            \includegraphics[width=\textwidth]{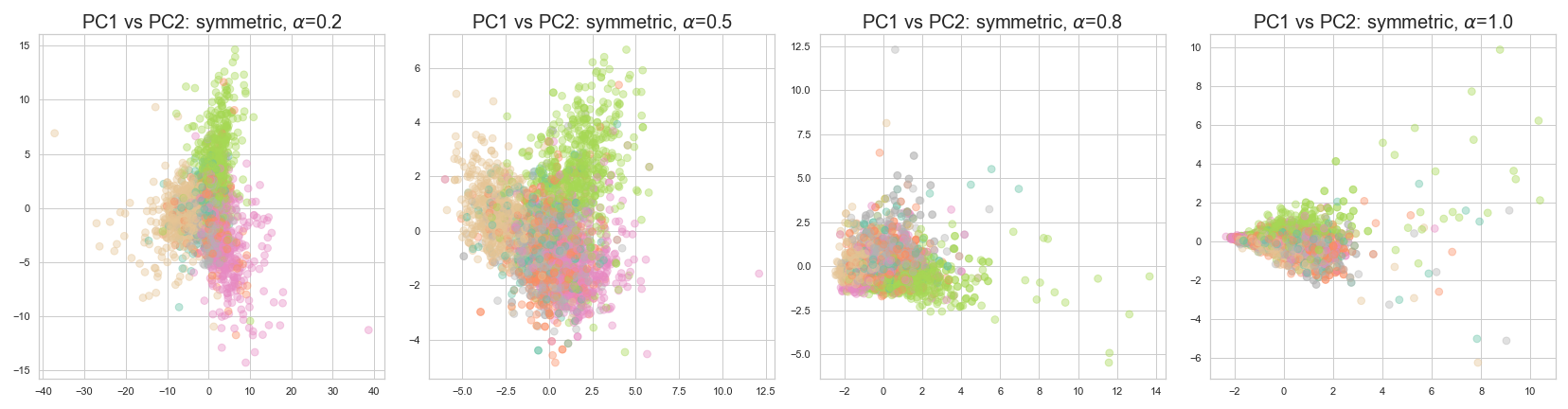}
            \caption{Citeseer, symmetrized, PCA, colored by node label}
        \end{subfigure}
        \begin{subfigure}{1\textwidth}
            \includegraphics[width=\textwidth]{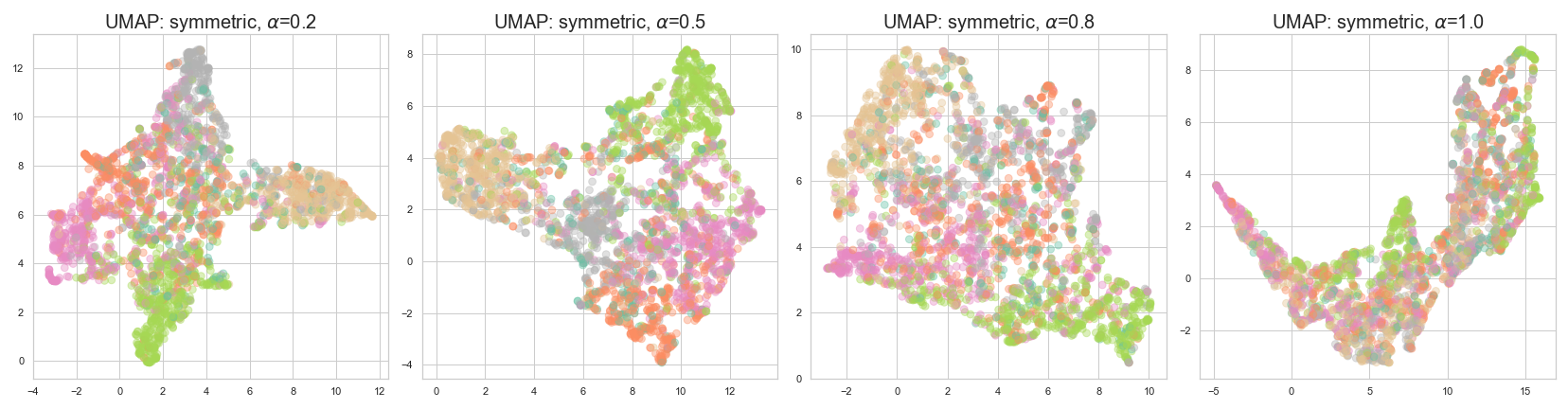}
            \caption{Citeseer, symmetrized, UMAP, colored by node label}
        \end{subfigure}
        \begin{subfigure}{1\textwidth}
            \includegraphics[width=\textwidth]{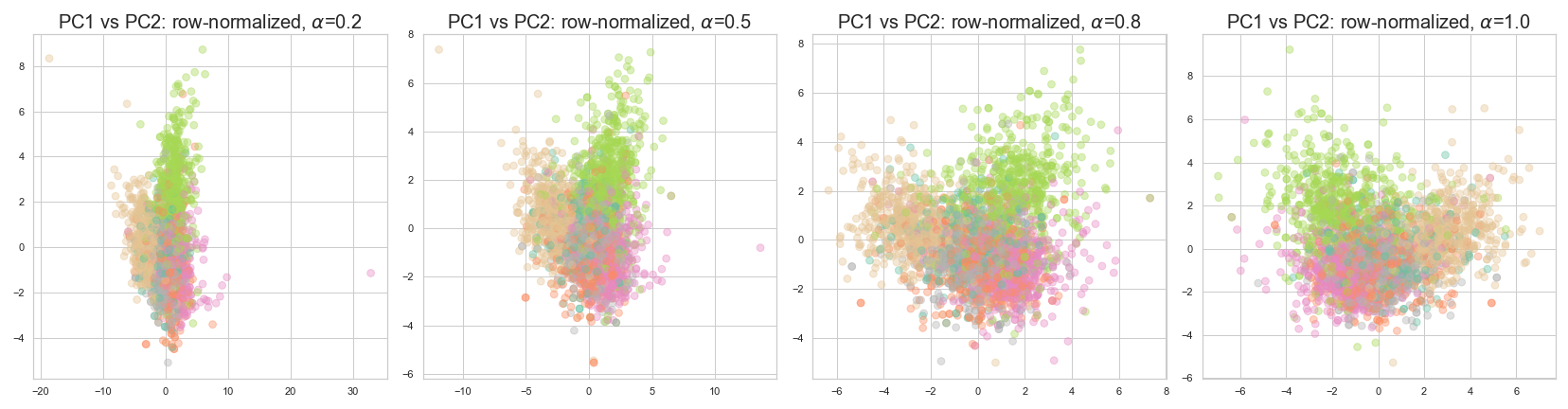}
            \caption{Citeseer, row-normalized, PCA, colored by node label}
        \end{subfigure}
        \begin{subfigure}{1\textwidth}
            \includegraphics[width=\textwidth]{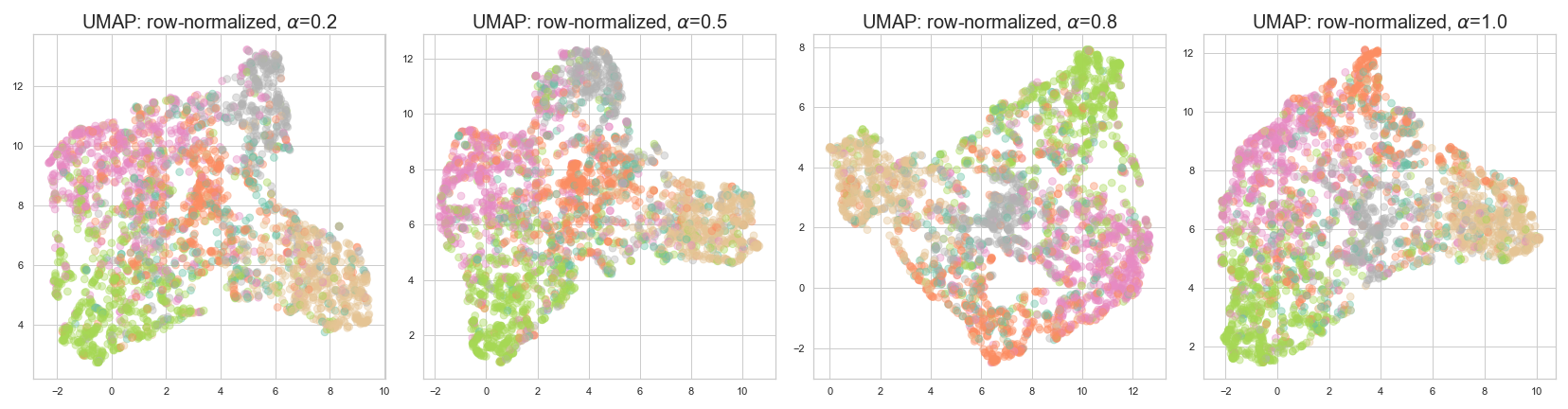}
            \caption{Citeseer, row-normalized, UMAP, colored by node label}
        \end{subfigure}
    \caption{Citeseer. The plots are colored by node labels(product categories). Embedding spaces generated by symmetric operator, (a), (b), as $\alpha$ increases the level of distinction between the cluster of different node labels decreases. Embedding space generated by row-normalized operator seems to be robust to the choice of $\alpha$-- it gives relatively constant level of clustering regardless of $\alpha$.}
    \label{fig:citeseer}
\end{figure}

\begin{figure}[H]
    \centering
        \begin{subfigure}{1\textwidth}
            \includegraphics[width=\textwidth]{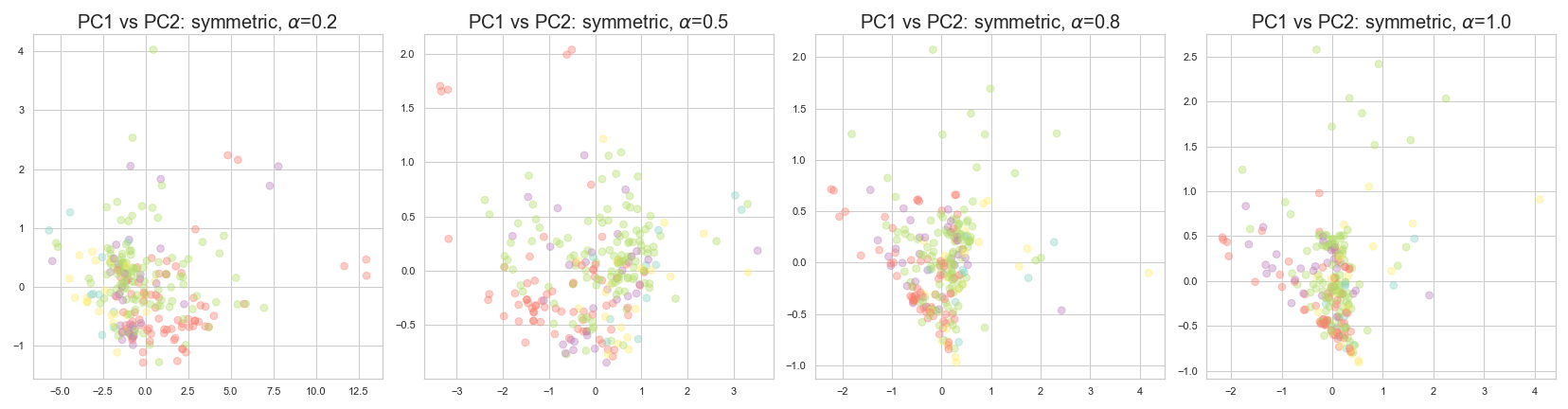}
            \caption{Wisconsin, symmetrized, PCA, colored by node label}
        \end{subfigure}
        \begin{subfigure}{1\textwidth}
            \includegraphics[width=\textwidth]{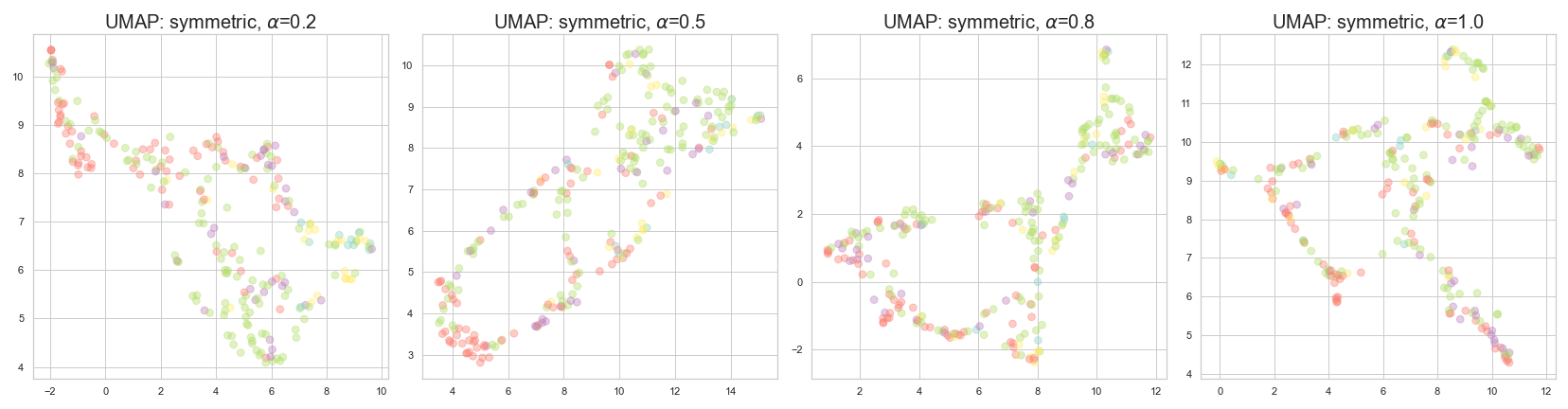}
            \caption{Wisconsin, symmetrized, UMAP, colored by node label}
        \end{subfigure}
        \begin{subfigure}{1\textwidth}
            \includegraphics[width=\textwidth]{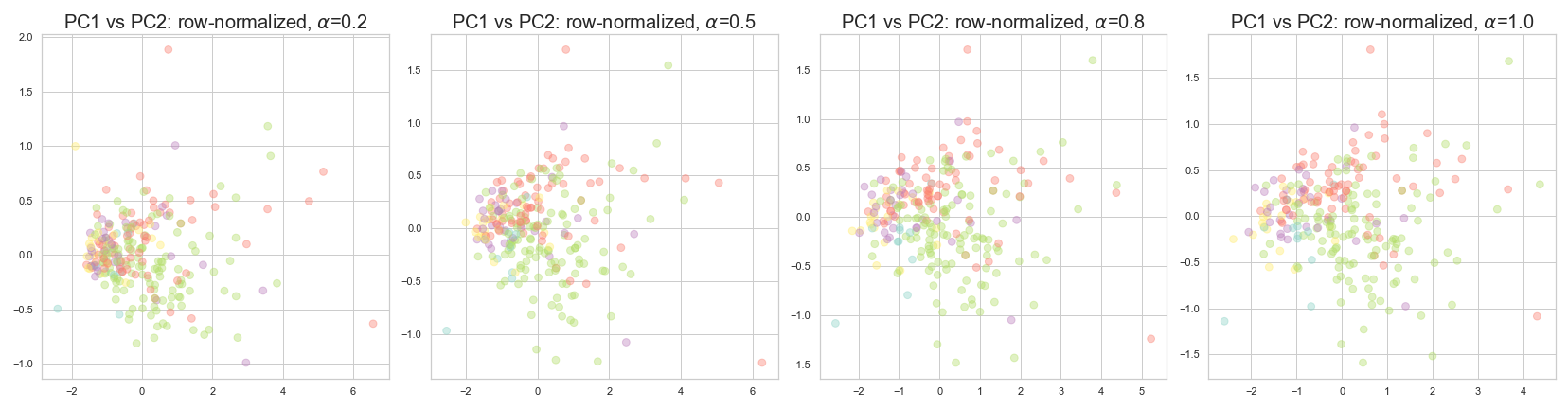}
            \caption{Wisconsin, row-normalized, PCA, colored by node label}
        \end{subfigure}
        \begin{subfigure}{1\textwidth}
            \includegraphics[width=\textwidth]{Experiment_Figures_Fin/classification/Wisconsin_embedding_normal_loop_True_UMAP_class.png}
            \caption{Wisconsin, row-normalized, UMAP, colored by node label}
        \end{subfigure}
    \caption{Wisconsin.The plots are colored by node label(categories for the webpage). Unlike the graph in Figure 8, it is hard to detect the change in the level of clustering or separation of each node class as $\alpha$ varies. PCA transformed embedding plot for row-normalized operator (c) even shows that the clustering of node label improves as $\alpha$ gets closer to 1.}
    \label{fig:wisc}
\end{figure}

\xhdr{Analysis} For standard homophilic datasets such as Citeseer(Figure \ref{fig:citeseer}), clustering of each node class has become less identifiable for a symmetric operator when $\alpha$ increases. On the other hand, the node class has been well separated across the alpha for the row-normalized operator-- the row-normalized operator is robust to the choice of $\alpha$ when it comes to node classification performance. Conversely, for the datasets with low homophily shown in Table \ref{table:char}, such as Wisconsin(Figure \ref{fig:wisc}), the separation of the node label does not change much depending on the choice of $\alpha$ or the choice of operator. The visual inspections on the embedding space transformed by PCA and UMAP are in line with the numerical result of test accuracy, Figure \ref{fig:node_class}.

 \label{sec:app-results-classification}
\begin{figure}[H]
    \centering
        \begin{subfigure}{.49\textwidth}
            \includegraphics[width=\textwidth]{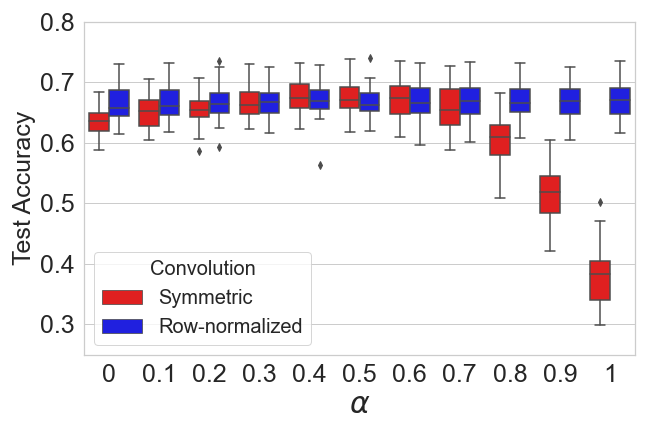}
            \caption{Citeseer}
        \end{subfigure}
        \begin{subfigure}{.49\textwidth}
            \includegraphics[width=\textwidth]{Experiment_Figures_Fin/Accuracy/wisconsin_accuracy.png}
            \caption{Wisconsin}
        \end{subfigure}
    \caption{Test accuracy for node classification task on two datasets: Citeseer and Wisconsin.}
    \label{fig:node_class}
\end{figure}

\subsubsection{Degree} \label{sec:app-results-degree}

In this subsection, we propose to investigate how basic topological characteristics (more specifically, here, the node degree) drive the organization of the embedding space. Consequently, to complement the analysis performed in the main text, we conduct visual inspection on the embedding plots of our benchmark datasets colored by node degree. 

Figures \ref{fig:co-author-cs-node-degree} and \ref{fig:amazon-node-degree} show the embedding spaces transformed by PCA, and the size and color of points denote the node degree. It is noted that the high degree nodes are marginalized when $\alpha$ close to 0, and lower degree nodes tend to be located at the origin. As $\alpha$ gets closer to 1, this pattern seems to be reverted--- the higher degree nodes are located at the origin and the lower degree nodes are pushed out to be at the periphery. 

\begin{figure}[H]
     \centering
     \begin{subfigure}{1\textwidth}
    \includegraphics[width=\textwidth]{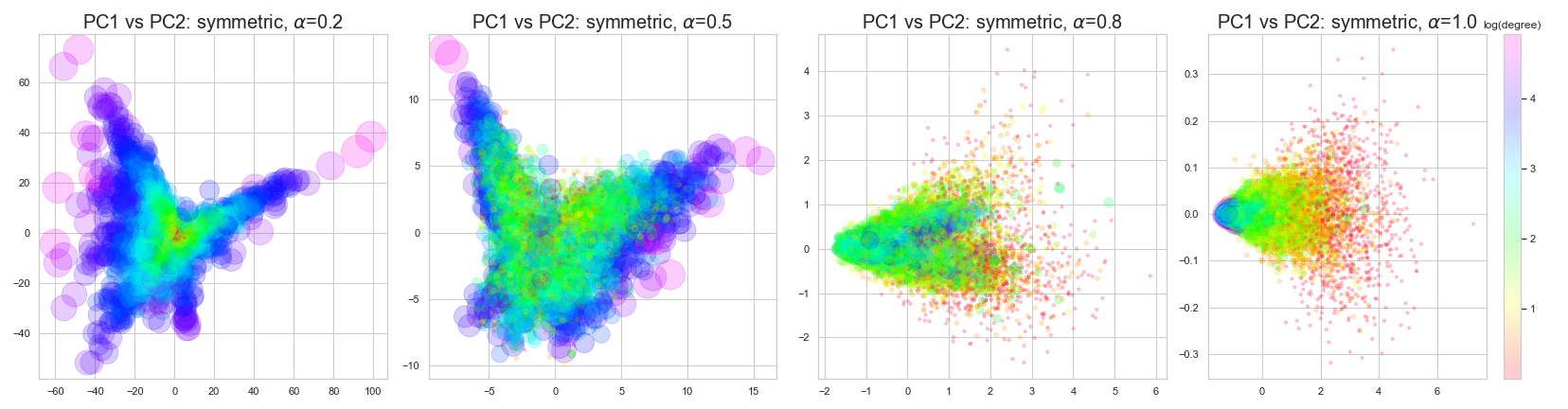}
    \caption{Coauthor CS, symmetrized, colored by degree}
    \end{subfigure}
    \begin{subfigure}{1\textwidth}
        \includegraphics[width=\textwidth]{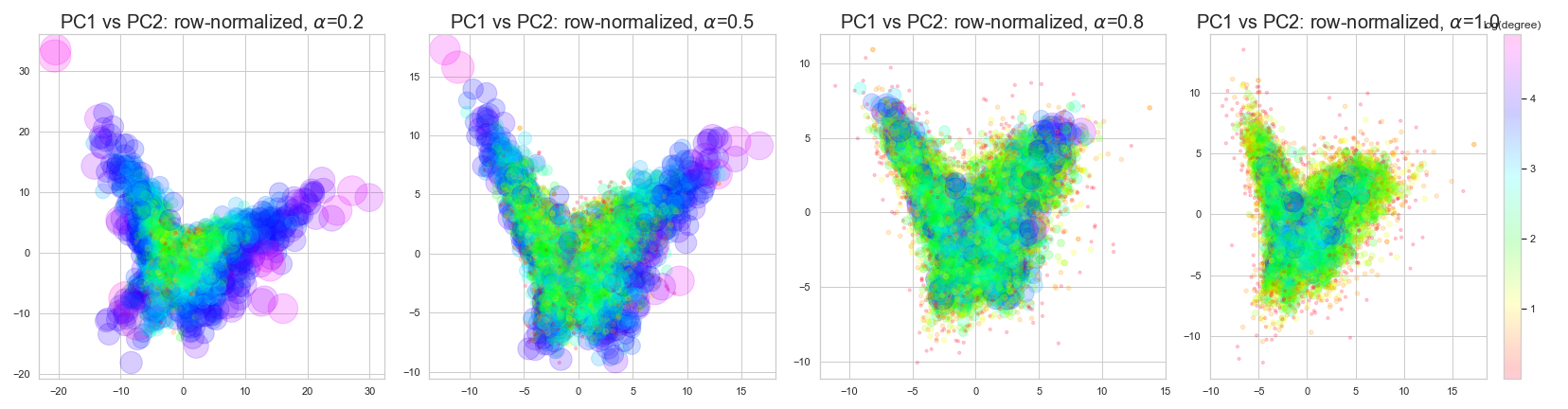}
        \caption{Coauthor CS, row-normalized, colored by degree}
    \end{subfigure}    
    \caption{Coauthor CS. The point size and color denote the node degree. For both symmetrized and row-normalized operator, high degree nodes are located farther from the origin when $alpha \approx 0$. As $\alpha$ increases, high degree nodes are concentrated on the origin, and low degree nodes are spread out instead.}
    \label{fig:co-author-cs-node-degree}
\end{figure}

\begin{figure}[H]
     \centering
     \begin{subfigure}{1\textwidth}
    \includegraphics[width=\textwidth]{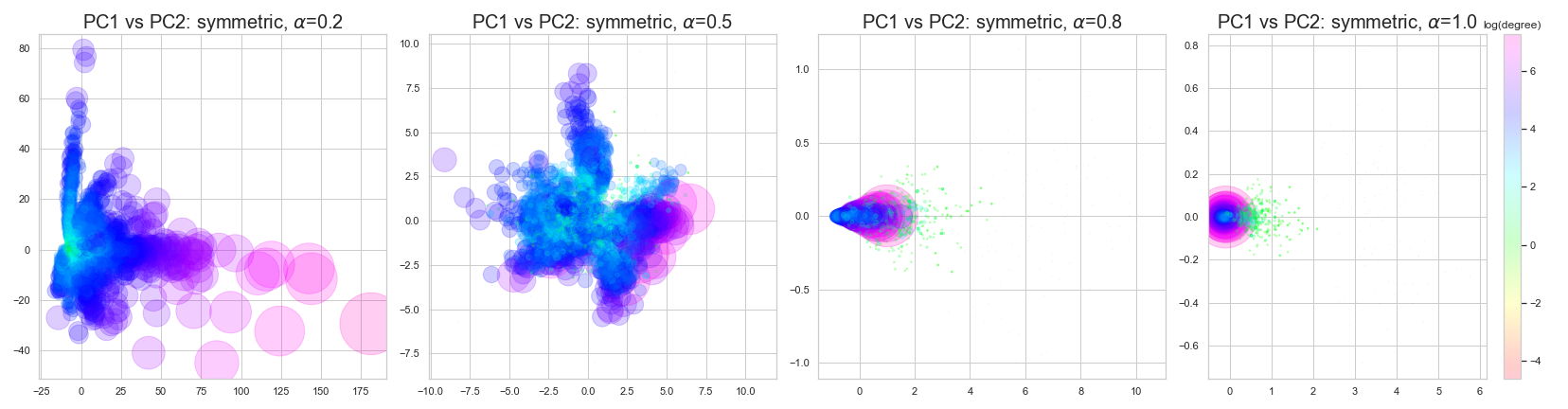}
    \caption{Amazon Photo, symmetrized, colored by degree}
    \end{subfigure}
    \begin{subfigure}{1\textwidth}
        \includegraphics[width=\textwidth]{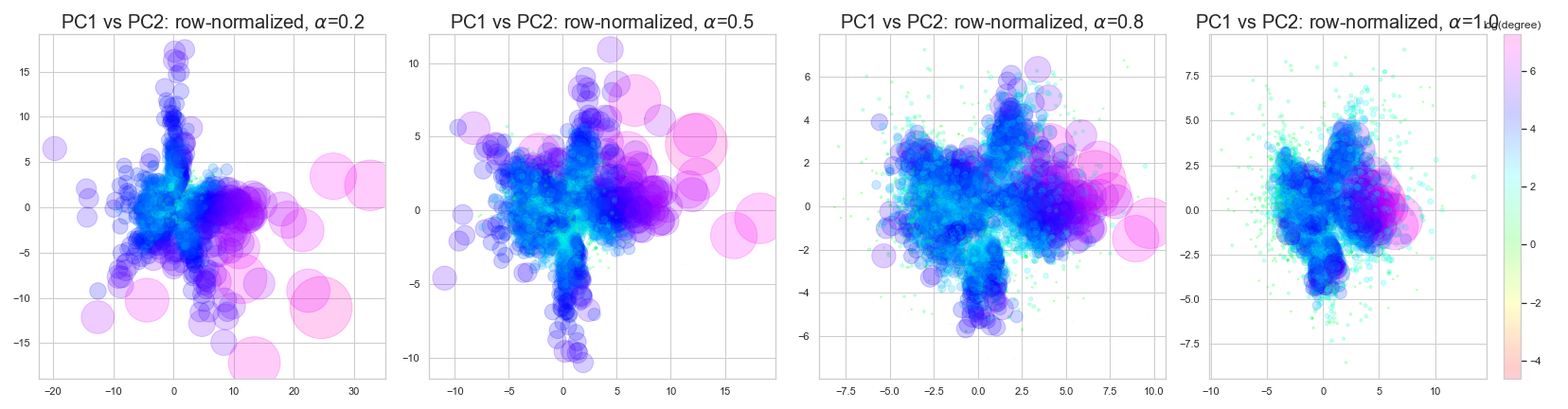}
        \caption{Amazon Photo, row-normalized, colored by degree}
    \end{subfigure}
    \caption{Amazon Photo. The point size and color denote the node degree. Amazon Photo networks have a few nodes with extremely high degree$(>500)$. Even for such nodes, as $\alpha$ increases the effect of high degree vanishes and all points are clustered near the origin.}
    \label{fig:amazon-node-degree}
\end{figure}

\subsubsection{Distance to the Original Space}\label{sec:app-results-distance}

In this subsection, we investigate the link between the relative distances between embedding points, and that of the original data.

\xhdr{Distance Calculation} The original dataset provides two separate views of the data, for which we can define two separate notions of distance: (1) a distance based on the graph structure (e.g the adjacency matrix), and (2) a distance based on the node features. For the distance in the graph, we choose to consider a distance in the graph space based on the diffusion distance \cite{coifman2006diffusion} using Gaussian kernel with $\epsilon = 0.5$. 
\begin{align*}
    K(u,v) = exp\big(-\frac{d_{\text{shortest path}
    }(\text{node u}, \text{node v})_{\alpha}^2}{\epsilon}\big)
\end{align*} 

The shortest path distance is computed by build-in function in \cite{networkx}. Distance in the feature space is measured by the pairwise euclidean distance of node features space. Finally, the distance in the embedding space is all based on the pairwise $\ell_2$ Euclidean distance. 

\xhdr{Correlation Analysis} It is a natural question to ask how embedding space closely resembles the original graph space or the feature space. The notion of closeness can be defined in several ways, but in this experiment, we first see the correlation between the distance in the original space and in the embedding space. We will use Spearman's rank correlation, which measures the monotonic relationship between the two.

Higher correlation could be interpreted as the amount of information that is retained in the embedding space regarding graph structure or node features. From Figure \ref{fig:distance-corr-graph}, both dataset show decreasing correlation as alpha increases; however, the correlation itself is actually close to 0($<0.05$). It is be reasonable to suspect no information regarding graph structure has been preserved in the embedding space, so we will come back to this question in \ref{sec:app-results-curvature}

Figure \ref{fig:distance-corr-feature} shows the trend in consonance with the test accuracy for the node classification task, Figure \ref{fig:node_class}. That is, for row-normalized operator the amount of node feature information is relatively constant for both standard homophilic dataset and heterophilic dataset such as Wisconsin. On the other hand, when it comes to the symmetrized operator, for Cora dataset, the correlation between the embedding spaces and the original feature space drops from 0.15(when $\alpha \approx 0.6$), to 0.005(when $\alpha \approx 1$). For Wisconsin dataset, the correlation still increases for the symmetrized operator, but the absolute value of correlation for the symmetrized operator is lower than that of the row-normalized operator.

\begin{figure}[H]
     \centering
     \begin{subfigure}{.49\textwidth}
    \includegraphics[width=\textwidth]{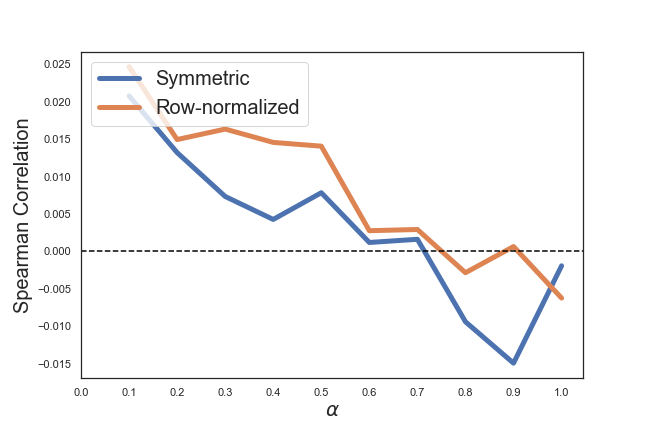}
    \caption{Cora}
    \end{subfigure}
      \begin{subfigure}{.49\textwidth}
    \includegraphics[width=\textwidth]{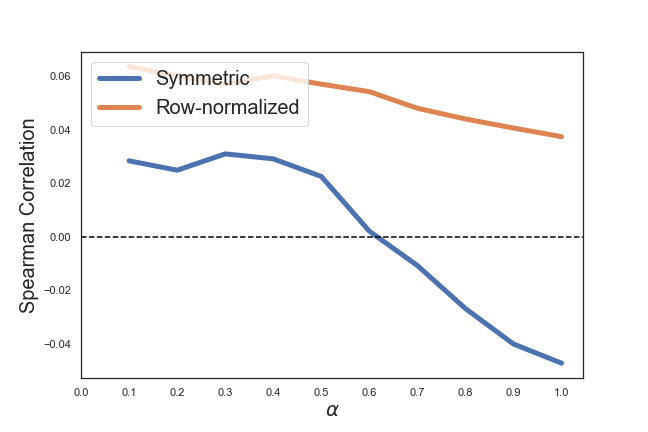}
    \caption{Wisconsin}
    \end{subfigure}
     \caption{Spearman's correlation between the pairwise distances in the graph space and pairwise distance in the embedding space. }
     \label{fig:distance-corr-graph}
\end{figure}

\begin{figure}[H]
    \begin{subfigure}{.49\textwidth}
    \includegraphics[width=\textwidth]{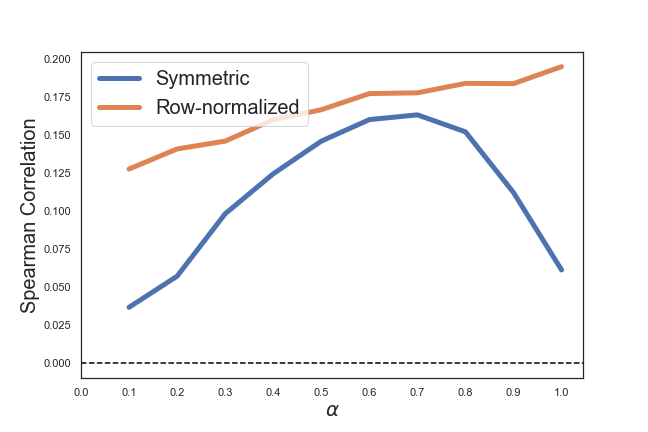}
    \caption{Cora}
    \end{subfigure}
    \begin{subfigure}{.49\textwidth}
    \includegraphics[width=\textwidth]{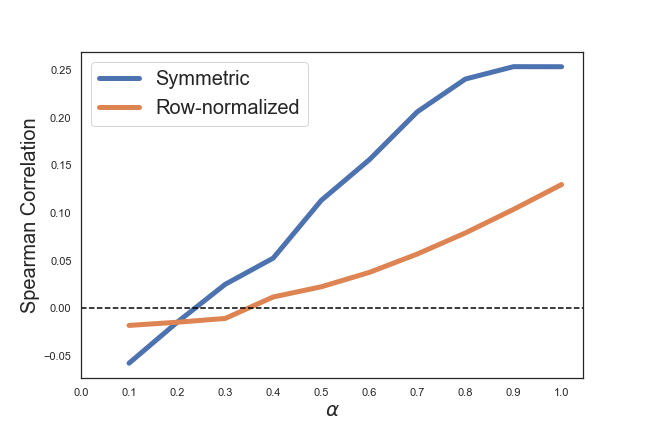}
    \caption{Wisconsin}
    \end{subfigure}
    \caption{Spearman's correlation between the pairwise distances in the feature space and pairwise distance in the embedding space}
    \label{fig:distance-corr-feature}
\end{figure}

\xhdr{Gromov-Wasserstein Distance} To this extent, we use Gromov-Wasserstein distance\cite{Memoli2011GW} which allows to measure the distance between two probability spaces of different dimensions, by comparing the within distance of probability spaces. By estimating Gromov-Wasserstein distance we can evaluate how close our embedding space is to the original space upon the choice of operators.

Based on the within distance calculated as described earlier, Gromov-wasserstein distance is calculated using python implemented \verb|ot.gromov.gromov.wasserstein| function in \verb|ot| package.\cite{flamary2021pot} \href{https://pythonot.github.io/quickstart.html}{\textit{https://pythonot.github.io}}. The detailed values from computations are shown in Figure \ref{fig:gw-graph} and Figure \ref{fig:gw-feature}.

\xhdr{Analysis} For the distance, we need to interpret in the opposite way we comprehend the correlation from earlier subsection. The lower the distance, the more the information regarding graph structure of feature has preserved in the embedding spaces. First, Cora shows the opposite pattern of distance with graph space and feature space. It can be viewed as for Cora, feature information has maximally preserved when $\alpha \approx 0.5$ for both symmetrized and row-normalized operator, while the information regarding graph structure has minimally estimated. When $\alpha \approx 0$ or $\alpha \approx 1$, the distance to the graph space is close to  0, while the distance to the feature space is close to the highest value.

On the contrary, Wisconsin seems to have similar pattern of distance for both graph and feature spaces. Embedding space recovered by the symmetrized operator, the information for both graph structure and node features are minimally retained when $\alpha \approx 0.5$. With the row-normalized operator, the distances to the original spaces increase as $alpha$ increases. 


\begin{figure}[H]
    \centering
        \begin{subfigure}[t]{0.49\textwidth}
            \includegraphics[width=\textwidth]{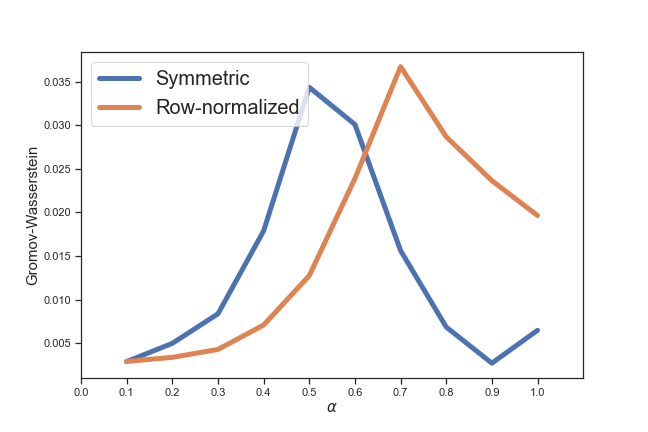}
            \caption{Cora}
        \end{subfigure}
        \begin{subfigure}[t]{0.49\textwidth}
            \includegraphics[width=\textwidth]{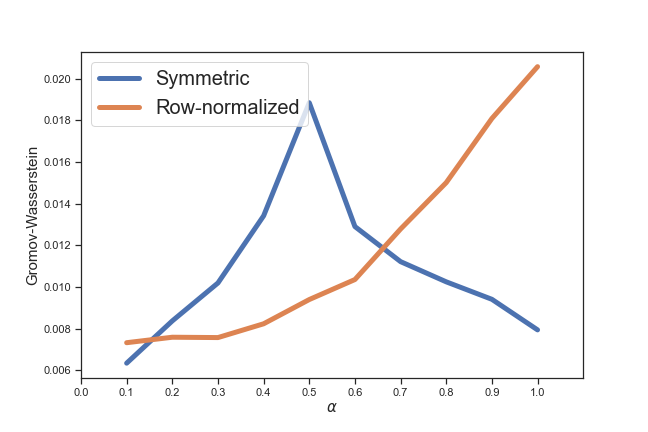}
            \caption{Wisconsin}
        \end{subfigure}
    \caption{Gromov-Wasserstein distance between the graph space and embedding space. For both datasets, when $\alpha$ is close to 0 or 1, the distance between two spaces is small for symmetrized operator. } 
    \label{fig:gw-graph}
\end{figure}

\begin{figure}[H]
    \centering
        \begin{subfigure}[t]{0.49\textwidth}
            \includegraphics[width=\textwidth]{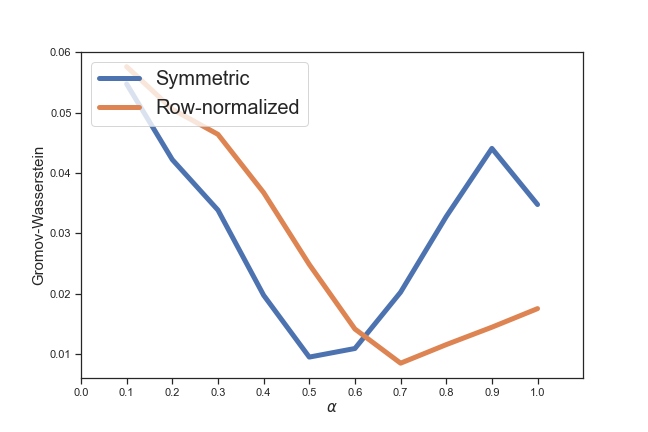}
            \caption{Cora}
        \end{subfigure}
        \begin{subfigure}[t]{0.49\textwidth}
            \includegraphics[width=\textwidth]{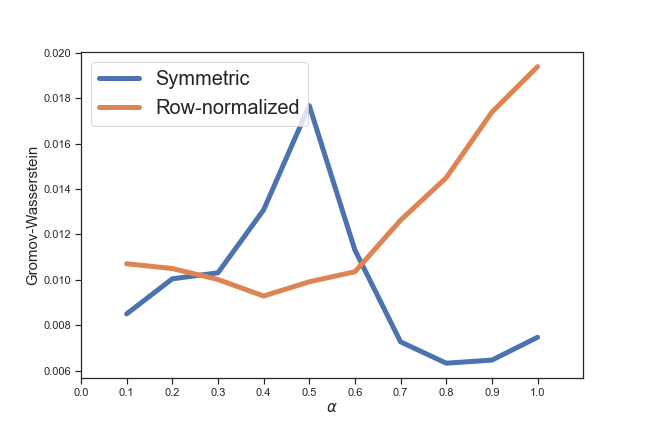}
            \caption{Wisconsin}
        \end{subfigure}
    \caption{Gromov-Wasserstein distance between the feature space and embedding space. Note that the distance variation across $\alpha$ is very similar to the accuracy of node classification task along $\alpha$ on Figure \ref{fig:node_class}
    }
    \label{fig:gw-feature}
\end{figure}

\subsubsection{Curvature} \label{sec:app-results-curvature}

In this section, the embedding space is compared to the original space with regard to the geometry of the original space. We first narrow down the notion of geometry to  a graph curvature. Graph curvature could explain the structural properties of the data that cannot be fully captured by the node degree. One might reasonably wonder how this structural information or geometry of the graph could be preserved from original space to the embedding space. We use augmented Forman curvature for the graph defined in \cite{Giovanni2022Curvature}. 

\begin{table}[H]
\centering
\resizebox{\textwidth}{!}{%
\begin{tabular}{@{}ccccccccc@{}}
\toprule
 &
  \textbf{Cora} &
  \textbf{Pubmed} &
  \textbf{Citeseer} &
  \textbf{\begin{tabular}[c]{@{}c@{}}Coauthor\\ CS\end{tabular}} &
  \textbf{\begin{tabular}[c]{@{}c@{}}Amazon\\ Photo\end{tabular}} &
  \textbf{Actor} &
  \textbf{Cornell} &
  \textbf{Wisconsin} \\ \midrule
\textbf{Mean} &
  -9.6178 &
  -18.9898 &
  -3.2427 &
  -14.4801 &
  -99.8552 &
  -8.2625 &
  -41.9855 &
  -46.2206 \\
\textbf{SD} &
  16.0352 &
  15.9152 &
  8.5414 &
  11.7537 &
  110.3011 &
  21.2629 &
  43.2669 &
  52.3953 \\ \bottomrule
\end{tabular}%
\caption{Mean and Standard deviation of augmented forman curvature\cite{Giovanni2022Curvature} for 8 Datasets}
}
\end{table}

To calculate the graph curvature on the embedding space, we have to reconstruct the graph on the embedding space. First, based on the euclidean distance of each node in the embedding space, we connect the same number of edges as the original graph. With this "reconstructed graph" on the embedding space, we calculate the graph curvature. Finally, we compare how much curvature has been preserved upon varying operators by Spearman's rank correlation. 

Figure \ref{fig:curvature-corr} shows (a) Cora and (b) Amazon Photos, standard datasets with high homophily as shown in Table 3, Spearman's correlation between the original curvature and embedding curvature are relatively constant around 0.3 with row-normalized operator. On the other hand, symmetrized operator has stronger positive correlation when $\alpha \approx 1$. For the dataset with low homophily, denoted as heterophilic graph dataset, such as (c) Cornell or (d) Wisconsin, not only the absolute value of the Spearman's correlation is much lower than that of results from homophilic dataset, but also there is a decreasing trend across the $\alpha$ for both symmetrized and row-normalized operator. 

Based on these observations, the geometry in terms of curvature seems to be better preserved when $\alpha \approx 1$ for the dataset with high homophily. When the graph is of low homophily, symmetrized operator works slightly better preserving the curvature, but the absolute value of the correlation itself is fairly low compared with the result of high homophily dataset, such as Figure \ref{fig:curvature-corr} (a) or (b).

\begin{figure}[H]
     \centering
     \begin{subfigure}{.49\textwidth}
    \includegraphics[width=\textwidth]{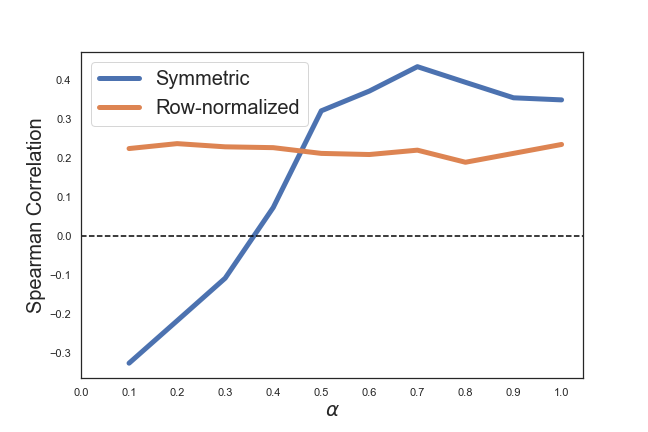}
    \caption{Cora}
    \end{subfigure}
    \begin{subfigure}{.49\textwidth}
    \includegraphics[width=\textwidth]{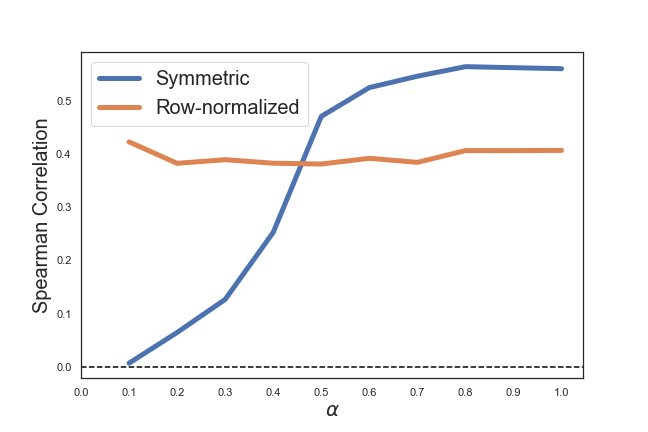}
    \caption{Citeseer}
    \end{subfigure}
     \begin{subfigure}{.49\textwidth}
    \includegraphics[width=\textwidth]{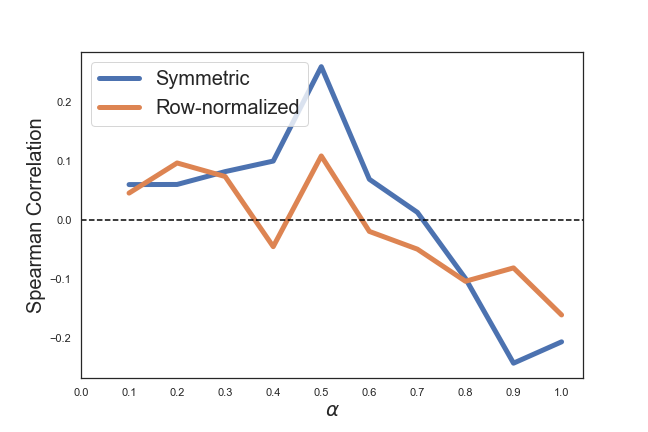}
    \caption{Cornell}
    \end{subfigure}
    \begin{subfigure}{.49\textwidth}
    \includegraphics[width=\textwidth]{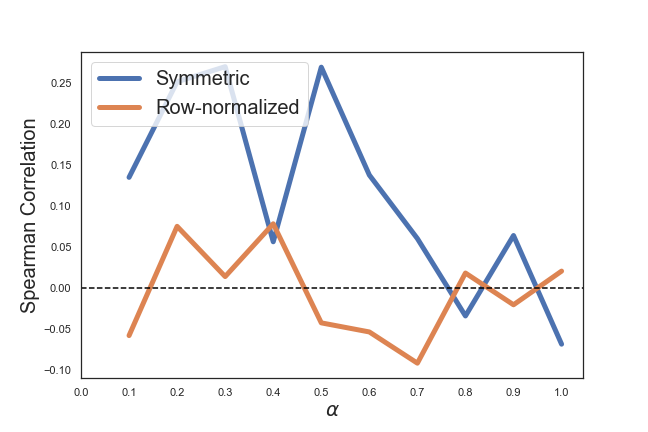}
    \caption{Wisconsin}
    \end{subfigure}
    \caption{Spearman's correlation between the original and embedding curvature. }
    \label{fig:curvature-corr}
\end{figure}

\begin{figure}[H]
     \centering
     \begin{subfigure}{1\textwidth}
    \includegraphics[width=\textwidth]{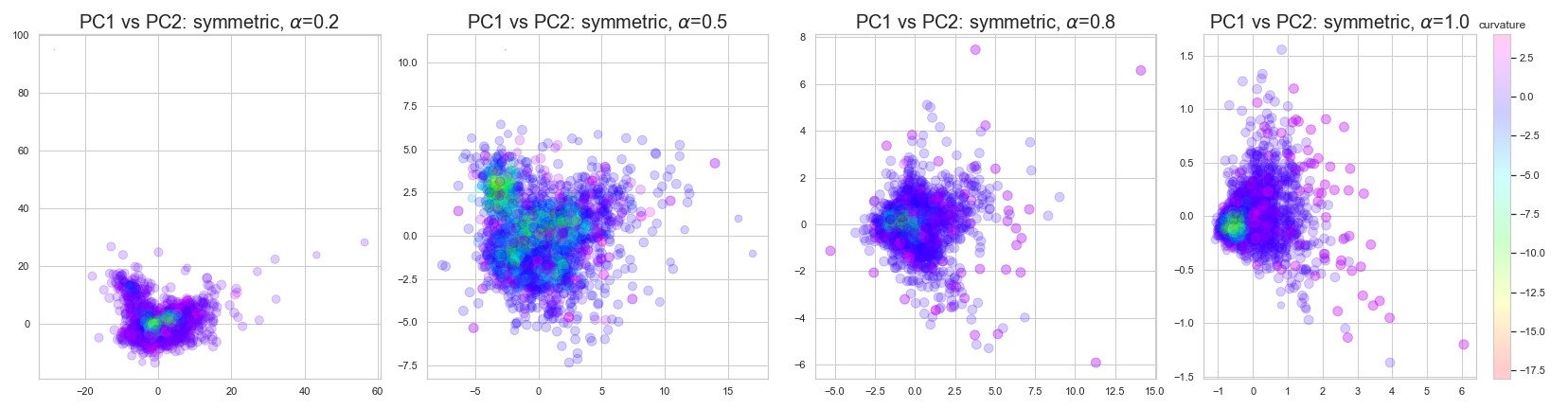}
    \caption{Cora, symmetrized, colored by embedding curvature}
    \end{subfigure}
     \begin{subfigure}{1\textwidth}
    \includegraphics[width=\textwidth]{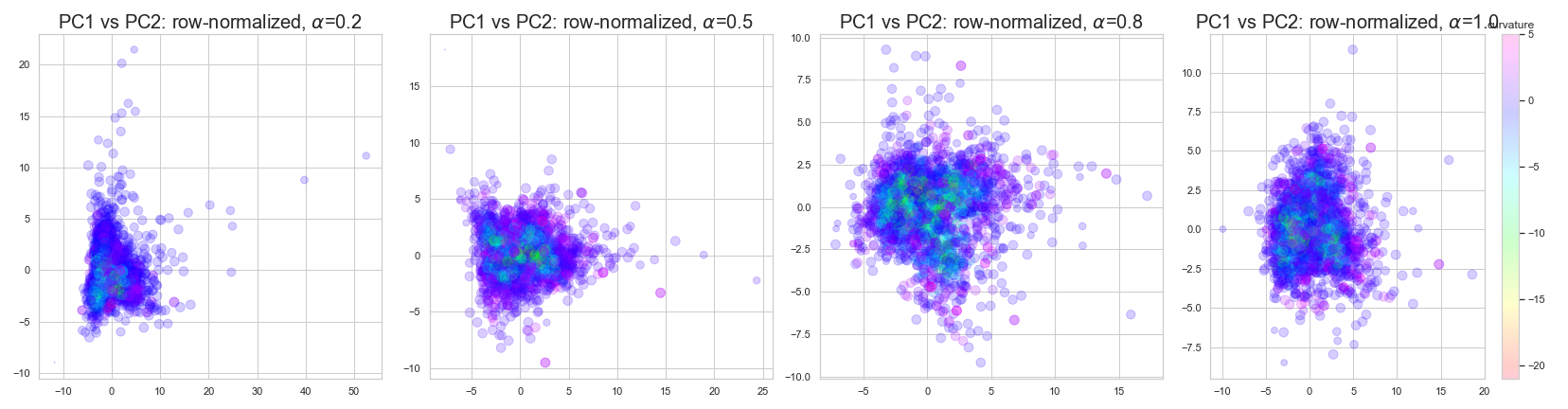}
    \caption{Cora, row-normalized, colored by embedding curvature}
    \end{subfigure}
    \label{fig:my_label}
    \caption{Cora. The points are colored by embedding curvature and the size is proportional to the original curvature.}
\end{figure}

\begin{figure}[H]
     \centering
     \begin{subfigure}{1\textwidth}
    \includegraphics[width=\textwidth]{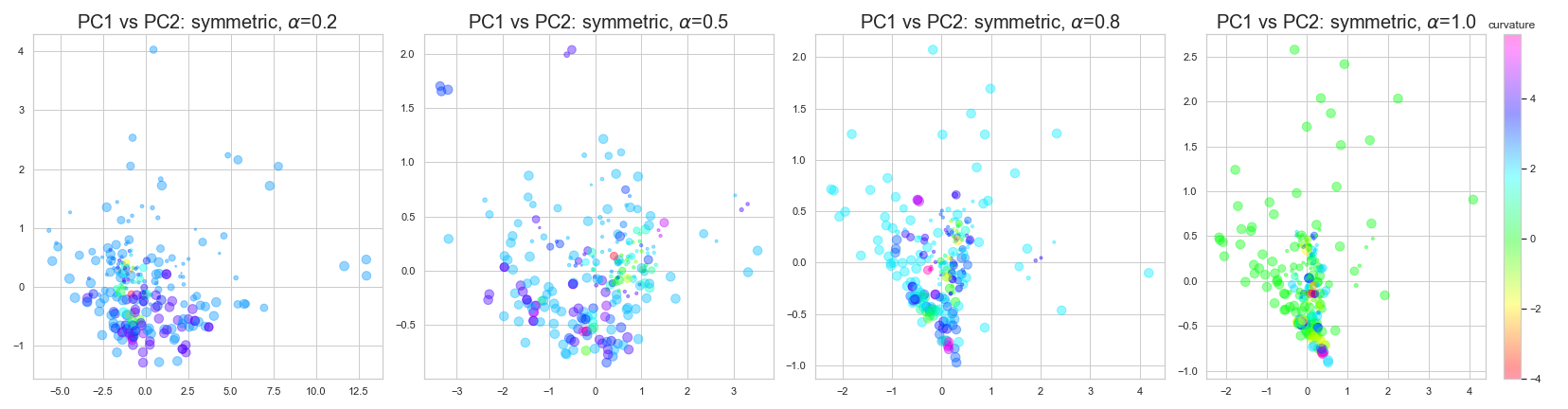}
    \caption{Wisconsin, symmetrized, colored by embedding curvature}
    \end{subfigure}
     \begin{subfigure}{1\textwidth}
    \includegraphics[width=\textwidth]{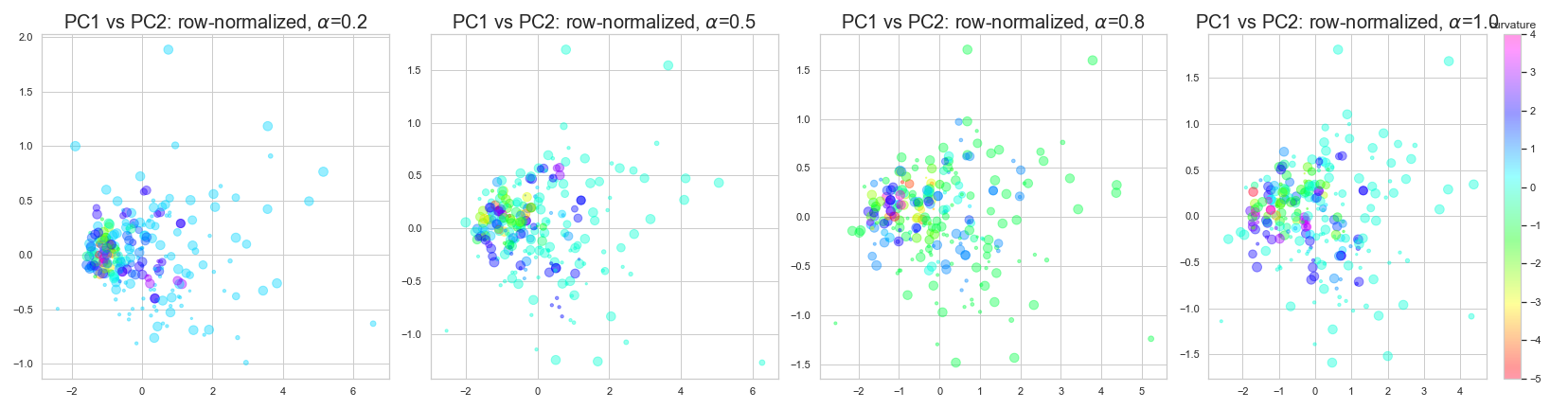}
    \caption{Wisconsin, row-normalized, colored by embedding curvature}
    \end{subfigure}
    \label{fig:my_label}
    \caption{Wisconsin. The points are colored by embedding curvature and the size is proportional to the original curvature.}
\end{figure}

\subsection{Additional Experimental results}\label{sec:app-additional}
\begin{figure}
        \centering
    \includegraphics[width=\textwidth]{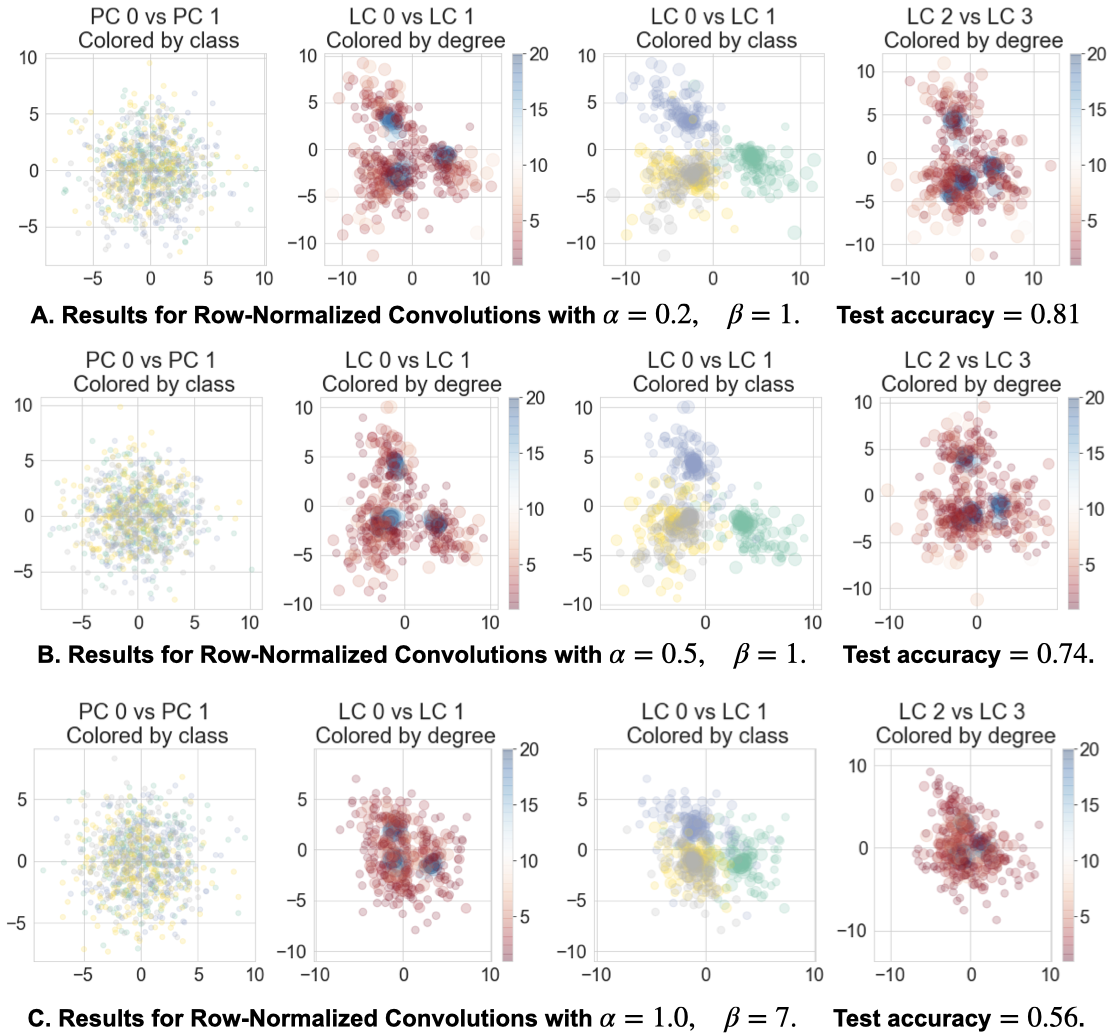}
    \caption{Row-Normalized Embeddings, plotted using the first two principal components (left), and the raw latent embeddings (or `latent components', shown in the right three plots on each row). Note the inversion: the high degree nodes migrate from the periphery of the latent space ($\a=0.2$) to the origin $(\a=0.7)$.}
    \label{fig:panel_regularised}
\end{figure}

\end{document}